\documentclass{svproc}

\usepackage{graphicx,subfigure}
\usepackage{amsmath}
\newcommand{\norm}[1]{\left\lVert#1\right\rVert}
\usepackage{amsfonts}
\usepackage{hyperref}
\hypersetup{colorlinks,allcolors=black}
\usepackage{url}

\usepackage{amsthm}
\usepackage{mathtools}
\usepackage{amssymb}
\usepackage{caption}
\newtheorem{mydef}{Definition}

\newcommand\js[1]{{\color{black}{#1}}}
\usepackage{cleveref}
\crefname{figure}{Fig.}{Fig.}
\Crefname{figure}{Figure}{Figure}
\crefname{equation}{}{}
\Crefname{equation}{Equation}{Equation}
\begin{document}
\mainmatter              % start of a contribution
\title{Deadlock Analysis and \\ Resolution for Multi-Robot Systems \\(Extended Version)}
\titlerunning{Deadlock Analysis for Multirobot Systems}  % abbreviated title (for running head)
\institute{The Robotics Institute, Carnegie Mellon University \\{\tt\small \{jaskarag,cliu6,sycara\}@andrew.cmu.edu}  }
\author{Jaskaran~Singh~Grover \and Changliu~Liu \and Katia~Sycara
\thanks{This research was supported by the DARPA Cooperative Agreement HR00111820051.}}
\date{}
\maketitle
\begin{abstract}
Collision avoidance for multirobot systems is a well studied problem. Recently, control barrier functions (CBFs) have been proposed for synthesizing controllers  guarantee collision avoidance and goal stabilization for multiple robots. However, it has been noted  reactive control synthesis methods (such as CBFs) are prone to \textit{deadlock}, an equilibrium of system dynamics  causes robots to come to a standstill before reaching their goals. In this paper, we formally derive characteristics of deadlock in a multirobot system  uses CBFs. We propose a novel approach to analyze deadlocks resulting from optimization based controllers (CBFs) by borrowing tools from duality theory and graph enumeration. Our key insight is  system deadlock is characterized by a force-equilibrium on robots and we show how complexity of deadlock analysis increases approximately exponentially with the number of robots. This analysis allows us to interpret deadlock as a subset of the state space, and we prove  this set is non-empty, bounded and located on the boundary of the safety set. Finally, we use these properties to develop a provably correct decentralized algorithm for deadlock resolution which ensures  robots converge to their goals while avoiding collisions. We show simulation results of the resolution algorithm for two and three robots and experimentally validate this algorithm on Khepera-IV robots.
\keywords{Collision Avoidance, Optimization and Optimal Control}
\end{abstract}
\section{Introduction}

Multirobot systems have been studied thoroughly for solving a variety of complex tasks such as search and rescue \cite{kantor2003distributed}, sensor coverage \cite{cortes2004coverage} and environmental exploration \cite{burgard2005coordinated}. Global coordinated behaviors result from executing local control laws on individual robots interacting with their neighbors \cite{ogren2002control}, \cite{olfati2007consensus}. Typically, the local controllers running on these robots are a combination of a task-based controller responsible for completion of a primary objective and a reactive collision avoidance controller. However, including a hand-engineered safety control no longer guarantees  the original task will be satisfied \cite{borrmann2015control}. 
This problem becomes all the more pronounced when the number of robots increases. Motivated by this bottleneck, our paper focuses on an algorithmic analysis of the performance-safety trade-offs  result from augmenting a task-based controller with collision avoidance constraints as done using CBF based quadratic programs (QPs) \cite{ames2017control}. Although CBF-QPs mediate between safety and performance in a rigorous way, yet ultimately they are distributed local controllers. Such approaches exhibit a lack of look-ahead, which causes the robots to be trapped in \textit{deadlocks} as noted in \cite{petti2005safe,o1989deadlockfree,wang2017safety}. In deadlock, the robots stop while still being away from their goals and persist in this state unless intervened.
\begin{figure}[t]
    \centering
    \includegraphics[trim={5.5cm 17.5cm 7.7cm 2.8cm},clip,width=.4\linewidth]{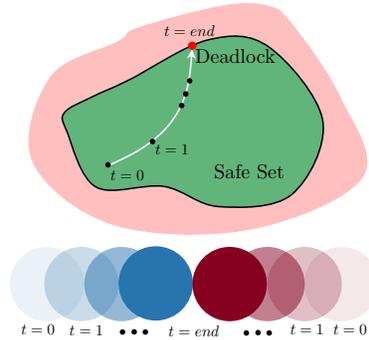}
    \setlength{\belowcaptionskip}{-15pt}
    \caption{Two robots moving towards each other fall in deadlock. System state converges to the boundary of  safe set.}
    \label{fig:deadlock_cartoon}
\end{figure}
This occurs because robots reach a state where conflict becomes inevitable, \textit{i.e.} a control favoring goal stabilization will violate safety (see red dot in \cref{fig:deadlock_cartoon}). Hence, the only feasible strategy is to remain static. Although small perturbations can steer the system away from deadlock, there is no guarantee  robots will not fall back in deadlock. To circumvent these issues, this work addresses the following technical questions:

\begin{enumerate}
    \item What are the characteristics of a system in deadlock ?
    \item What all geometric configurations of robots are admissible in deadlock ?
    \item How can we leverage this information to provably exit deadlock using decentralized controllers?
\end{enumerate}

To address these questions, we first review technical definitions for CBF based QPs \cite{wang2017safety}  to synthesize controllers for collision avoidance and goal stabilization  in \cref{CBFReview}. In  \cref{Analysis for deadlock}, we recall the definition of deadlock and use KKT conditions to motivate a novel set theoretic interpretation of deadlock with an eye towards devising controllers  evade/exit this set. We use graph enumeration to highlight the combinatorial complexity of geometric configurations of robots admissible in deadlock. Following this development, in \cref{Analysis for two robots} and \cref{Analysis for three robots}, we focus on the easier to analyze cases for two and three robots respectively and examine mathematical properties of the deadlock set for these cases. We show  this set is on the boundary of the safety set, is non-empty and bounded. In \cref{deadlock resolution}, we show how to design a provably-correct decentralized controller to make the robots exit deadlock. We demonstrate this strategy on two and three robots in simulation, and experimentally on Khepera-IV robots. Finally, we conclude with directions for future work.
\section{Prior Work}

\label{PriorWork}
Several existing  methods provide  inspiration  for  the  results  presented  here. Of these, two are especially relevant: in the first category, we describe prior methods for collision avoidance and in the second, we focus on deadlock resolution.
\subsection{Prior Work on Avoidance Control}

Avoidance control is a well-studied problem with immediate applications for planning collision-free motions for multirobot systems. Classical avoidance control assumes a worst case scenario with no cooperation between robots \cite{leitmann1977avoidance,leitmann1983note}
Cooperative collision avoidance is explored in \cite{stipanovic2007cooperative,hokayem2010coordination} where avoidance control laws are computed using value functions.  Velocity obstacles have been proposed in \cite{fiorini1998motion} for motion planning in dynamic environments. They select avoidance maneuvers outside of robot's velocity obstacles to avoid static and moving obstacles by means of a tree-search. While this method is prone to undesirable oscillations, the authors in \cite{van2008reciprocal,van2011reciprocal,wilkie2009generalized} propose reciprocal velocity obstacles  are immune to such oscillations. More recently, control barrier function based controllers have been used in \cite{borrmann2015control,wang2017safety} to mediate between safety and performance using QPs.
\subsection{Prior Work on Deadlock Resolution}

The importance of coordinating motions of multiple robots while simultaneously ensuring safety, performance and deadlock prevention has been acknowledged in works as early as in \cite{o1989deadlockfree}. Here, authors proposed scheduling algorithms to asynchronously coordinate motions of two manipulators to ensure  their trajectories remain collision and deadlock free.  In the context of mobile robots, \cite{yamaguchi1999cooperative} identified the presence of deadlocks in a cooperative  scenario using mobile robot troops. To the best of our knowledge, \cite{jager2001decentralized} were the first to propose algorithms for deadlock resolution specifically for multiple mobile robots. Their strategy for collision avoidance modifies planned paths by inserting idle times and resolves deadlocks by asking the trajectory planners of each robot to plan an alternative trajectory until deadlock is resolved. Authors in \cite{li2005motion} proposed coordination graphs to resolve deadlocks in robots navigating through narrow corridors. \cite{wang2017safety,rodriguez2016guaranteed} added perturbation terms to their controllers for avoiding deadlock. 
 \\ 
\hspace{1cm}Differently from these, we characterize analytical properties of system states when in deadlock. We explicitly analyze controls from CBF based QPs and demonstrate  intuitive explanations for systems in deadlock are indeed recovered using duality. Our analysis can be extended to reveal bottlenecks of any optimization based controller synthesis method. Additionally, we use graph enumeration to highlight the complexity of this analysis. We do not consider additive perturbations for resolving deadlocks, since there are no formal guarantees. Instead, we use feedback linearization and the geometric properties recovered from duality to guide the design of a provably correct controller  ensures safety, performance and deadlock resolution.
\section{Avoidance Control with CBFs: Review}

\label{CBFReview}
In this section, we review CBF based QPs used for synthesizing controllers  mediate between safety (collision avoidance) and performance (goal-stabilization) for multirobot systems. We refer the reader to \cite{wang2017safety} for a comprehensive treatment on this subject, since our work builds on top of their approach.  Assume  we have $N$ mobile robots, each of which follows double-integrator dynamics:

\begin{align}
\left[\begin{matrix}
\dot{\boldsymbol{p}}_i \\
\dot{\boldsymbol{v}}_i
\end{matrix}\right] &= \left[\begin{matrix}
\dot{\boldsymbol{v}}_i \\
\boldsymbol{0}
\end{matrix}\right] + \left[\begin{matrix}
\boldsymbol{\mbox{O}} \\
\mathbf{I}
\end{matrix}\right]\boldsymbol{u}_i, 
\end{align}
where $\boldsymbol{p}_i=(x_i,y_i)\in\mathbb{R}^2$ represents the position of robot $i$, $\boldsymbol{v}_i\in\mathbb{R}^2$ represents its velocity and $\boldsymbol{u}_i\in\mathbb{R}^2$ represents the acceleration (\textit{i.e.} control). The collective state of robot $i$ is denoted by $\boldsymbol{z}_i=(\boldsymbol{p}_i,\boldsymbol{v}_i)$ and the collective state of the multirobot system is denoted as $\boldsymbol{Z}=(\boldsymbol{z}_1,\boldsymbol{z}_2,\dots,\boldsymbol{z}_N)$.  Assume  each robot has maximum allowable acceleration limits $\vert\boldsymbol{u}_i\vert\leq\alpha_i$  represent actuator constraints. The problem of goal stabilization with avoidance control requires  each robot $i$ must reach a goal $\boldsymbol{p}_{d_i}$ while avoiding collisions with every other robot $j\neq i$. For reaching a goal, assume  there is a prescribed PD controller $\hat{\boldsymbol{u}}_i(\boldsymbol{z}_i)=-k_p(\boldsymbol{p}_i-\boldsymbol{p}_{d_i})-k_v\boldsymbol{v}_i$ with $k_p$,$k_v>0$. This controller is chosen as a nominal reference controller because by itself, it ensures exponential stabilization of each robot to its goal. However, there is no guarantee  the resulting trajectories will be collision free.

Based on \cite{wang2017safety}, a safety constraint is formulated for every pair of robots to ensure mutually collision free motions. This constraint is mathematically posed by defining a function  maps the joint state space of robots $i\mbox{ and }j$ to a real-valued safety index \textit{i.e.} $h:\mathbb{R}^4 \times \mathbb{R}^4\longrightarrow\mathbb{R}$. For a desired safety margin distance $D_s$, this index is defined as  
\begin{align}
\label{hdef}
h_{ij}=\sqrt{2(\alpha_i + \alpha_j)(\norm{\Delta \boldsymbol{p}_{ij}}-D_{s})} + \frac{\Delta \boldsymbol{p}^T_{ij}\Delta \boldsymbol{v}_{ij}}{\norm{\Delta\boldsymbol{p}_{ij}}}.
\end{align}
% If the robots follow single-integrator dynamics, we modify this function to $h_{ij} = \norm{\Delta \boldsymbol{p}_{ij}}^2-D^2_{s}$. 
Robots $i\mbox{ and }j$ are considered to be collision-free if their states $(\boldsymbol{z}_i,\boldsymbol{z}_j)$ are such  $h_{ij}(\boldsymbol{z}_i,\boldsymbol{z}_j)\geq0$. We define ``safe set" as the 0-level superset of $h_{ij}$ \textit{i.e.}  $\mathcal{C}_{ij}\coloneqq \{(\boldsymbol{z}_i,\boldsymbol{z}_j)\in\mathbb{R}^8\mid h_{ij}(\boldsymbol{z}_i,\boldsymbol{z}_j)\geq0$\}. The boundary of the safe set is 
\begin{align}
\label{safetysetdef}
    \partial \mathcal{C}_{ij}=\{(\boldsymbol{z}_i,\boldsymbol{z}_j) \in \mathbb{R}^4 \vert h(\boldsymbol{z}_i,\boldsymbol{z}_j)=0\}
\end{align}
Assuming  the initial positions of robots $i\mbox{ and }j$ are in the safe set $\mathcal{C}_{ij}$, we would like to synthesize controls $\boldsymbol{u}_i \mbox{ and }\boldsymbol{u}_j$  ensure  future states of the robots $i\mbox{ and }j$ also stay in $\mathcal{C}_{ij}$. This can be achieved by ensuring 
\begin{align}
\label{constraint}
\frac{dh_{ij}}{dt} \geq -\kappa(h_{ij}),
\end{align}
where we choose, $\kappa(h)\coloneqq h^3$. For the given choice of $h$, \cref{constraint} can be rewritten as 
\begin{align}
-\Delta\boldsymbol{p}^T_{ij}\Delta \boldsymbol{u}_{ij} \leq b_{ij}, \mbox{   where} 
\end{align}

\begin{align}
\label{safetyconstraint}
b_{ij} = 
&\norm{\Delta \boldsymbol{p}_{ij}}h^3_{ij} 
+ \frac{(\alpha_i + \alpha_j)\Delta \boldsymbol{p}^T_{ij}\Delta \boldsymbol{v}_{ij}}{\sqrt{2(\alpha_i + \alpha_j)(\norm{\Delta \boldsymbol{p}_{ij}}-D_{s})}}+ \norm{\Delta \boldsymbol{v}_{ij}}^2 - \frac{(\Delta \boldsymbol{p}^T_{ij}\Delta \boldsymbol{v}_{ij})^2}{\norm{\Delta \boldsymbol{p}_{ij}}^2}
\end{align}
This constraint is distributed on robots $i \mbox{ and }j$ as:

\begin{align}
\label{decentralizedconstraints}
-\Delta\boldsymbol{p}_{ij}^T\boldsymbol{u}_i \leq \frac{\alpha_i}{\alpha_i+\alpha_j}b_{ij} \mbox{     and     }
\Delta\boldsymbol{p}_{ij}^T\boldsymbol{u}_j \leq \frac{\alpha_j}{\alpha_i+\alpha_j}b_{ij}  
\end{align} \\ 
Therefore, any $\boldsymbol{u}_i$ and $\boldsymbol{u}_j$  satisfy \cref{decentralizedconstraints} will ensure collision free trajectories for robots $i\mbox{ and }j$ in the multirobot system.  Note  these constraints are linear in $\boldsymbol{u}_i$ and $\boldsymbol{u}_j$ for a given state $(\boldsymbol{z}_i,\boldsymbol{z}_j)$. Therefore, the feasible set of controls is convex. Assuming robot $i$ wants to avoid collisions with its $M$ neighbors, there will be $M$ collision avoidance constraints. To mediate between safety and goal stabilization, a QP is posed  computes a controller closest to the PD control $\hat{\boldsymbol{u}}_i(\boldsymbol{z}_i)$ (in 2-norm) and satisfies the $M$ collision avoidance constraints: 
\begin{equation}
\label{optimization_formulation_2}
	\begin{aligned}
		& \underset{\boldsymbol{u}_i}{\text{minimize}}
		& & \norm{\boldsymbol{u}_i - \hat{\boldsymbol{u}}_i(\boldsymbol{z}_i)}^2_2 \\
		& \text{subject to}
		& & -\Delta\boldsymbol{p}_{ij}^T\boldsymbol{u}_i \leq \frac{\alpha_i}{\alpha_i+\alpha_j}b_{ij}  \mbox{ } j\in \{1,\dots,M\} \\
		& & & \vert \boldsymbol{u}_i \vert \leq \boldsymbol{\alpha}_i
	\end{aligned}
\end{equation}\\
This QP has $(M+4)$ constraints ($M$ from collision avoidance with $M$ neighbors and four from acceleration limits). Each robot $i$ executes a local version of this QP and computes its optimal $\boldsymbol{u}^*_i$ at every time step. As long as the QP remains feasible, the generated control $\boldsymbol{u}^*_i$ ensures collision avoidance of robot $i$ with its neighbors. In the next section, we derive an analytical expression for $\boldsymbol{u}^*_i$ as a function of $(\boldsymbol{z}_1,\dots,\boldsymbol{z}_N)$ to analyze the closed-loop dynamics of the ego robot and use this to investigate the incidence of deadlocks resulting from this technique.

\section{Analysis of $N$ Robot Deadlock}
\label{Analysis for deadlock}

We reviewed the formulation of multirobot collision avoidance and goal stablization using the framework of CBF based QPs.  In this section, we will show  this approach can result in deadlocks (depending on the initial conditions and goals of robots).  We want to analyze qualitative properties of a robot in deadlock. Towards  end, we will investigate the KKT conditions  \cite{boyd2004convex} of the problem in \cref{optimization_formulation_2}. Our goal is to use these conditions to compute properties of geometric configurations of robots in deadlock and then exploit these properties to make the robots exit deadlock.  Fig. \ref{fig:deadlock_traj} shows the states of two robots  have fallen in deadlock while executing controllers based on \cref{optimization_formulation_2}. Notice from \cref{fig:deadlock_pos}  the positions of robots have converged, but \textbf{not} to their respective goals. Therefore, the outputs from the prescribed PD controller will still be non-zero after convergence. However, the control inputs from \cref{optimization_formulation_2} have already converged to zero \cref{fig:deadlock_acc}.  From these observations, deadlock is defined as follows  \cite{wang2017safety}
\begin{mydef}
\label{deadlock_definition}
A robot $i$ is in deadlock if $\boldsymbol{u}^*_i=0$, $\boldsymbol{v}_i=0$, $\hat{\boldsymbol{u}}_i \neq 0$ and $\boldsymbol{p}_i \neq \boldsymbol{p}_{d_i}$
\end{mydef}
In simpler terms, for a robot to be in deadlock, it should be static \textit{i.e.} its velocity should be zero,  and the output from the QP based controller should also be zero, even though the reference PD controller reports non-zero acceleration since the robot is not at its intended goal. 
We now look at the KKT conditions for the optimization problem in  \cref{optimization_formulation_2}.
\begin{figure}[t]
\setlength{\belowcaptionskip}{-12pt}
\centering     %%% not \center
\subfigure[Positions ]{\label{fig:deadlock_pos}\includegraphics[width=58mm]{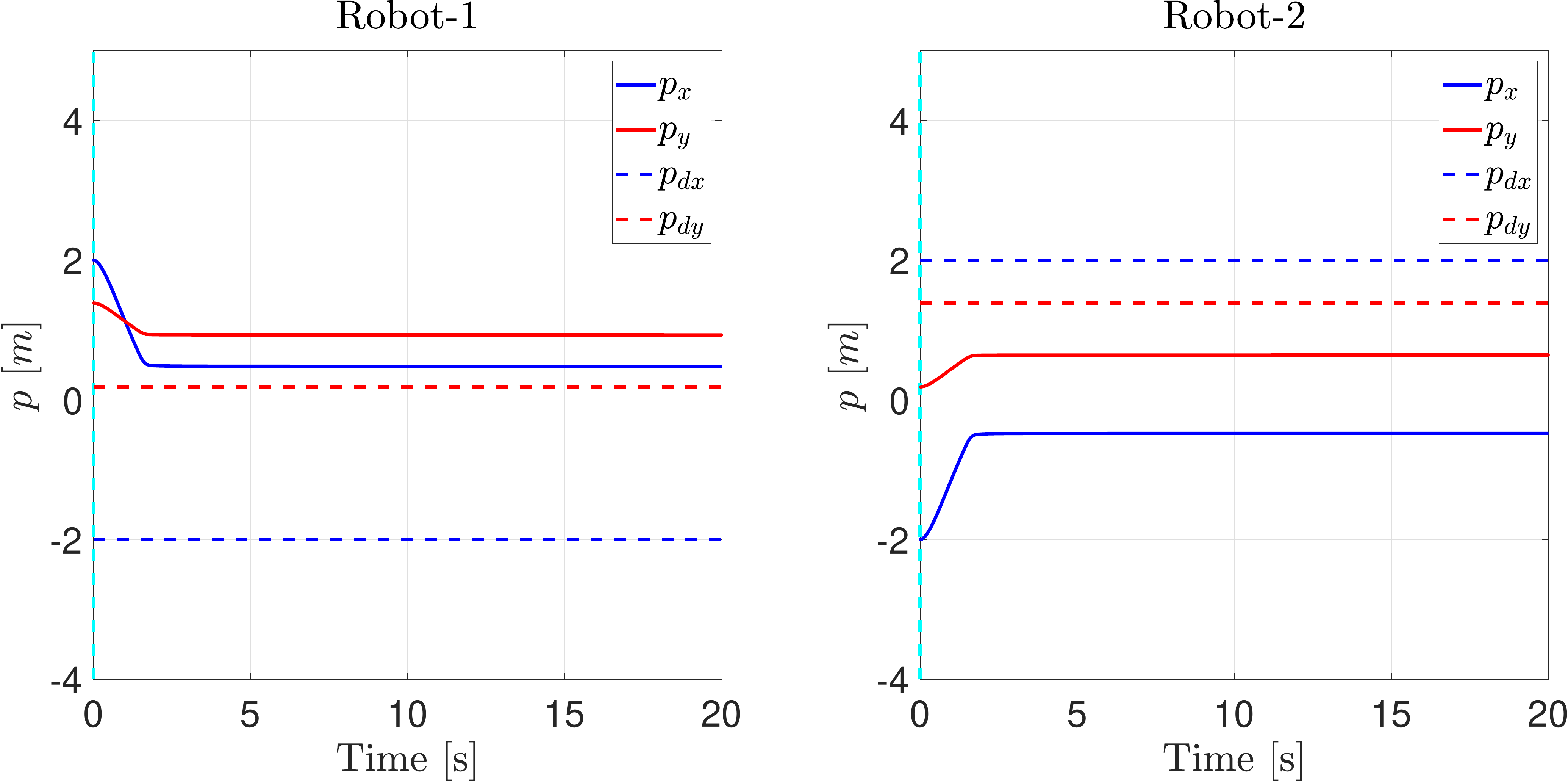}}
\subfigure[Accelerations (control inputs) ]{\label{fig:deadlock_acc}\includegraphics[width=58mm]{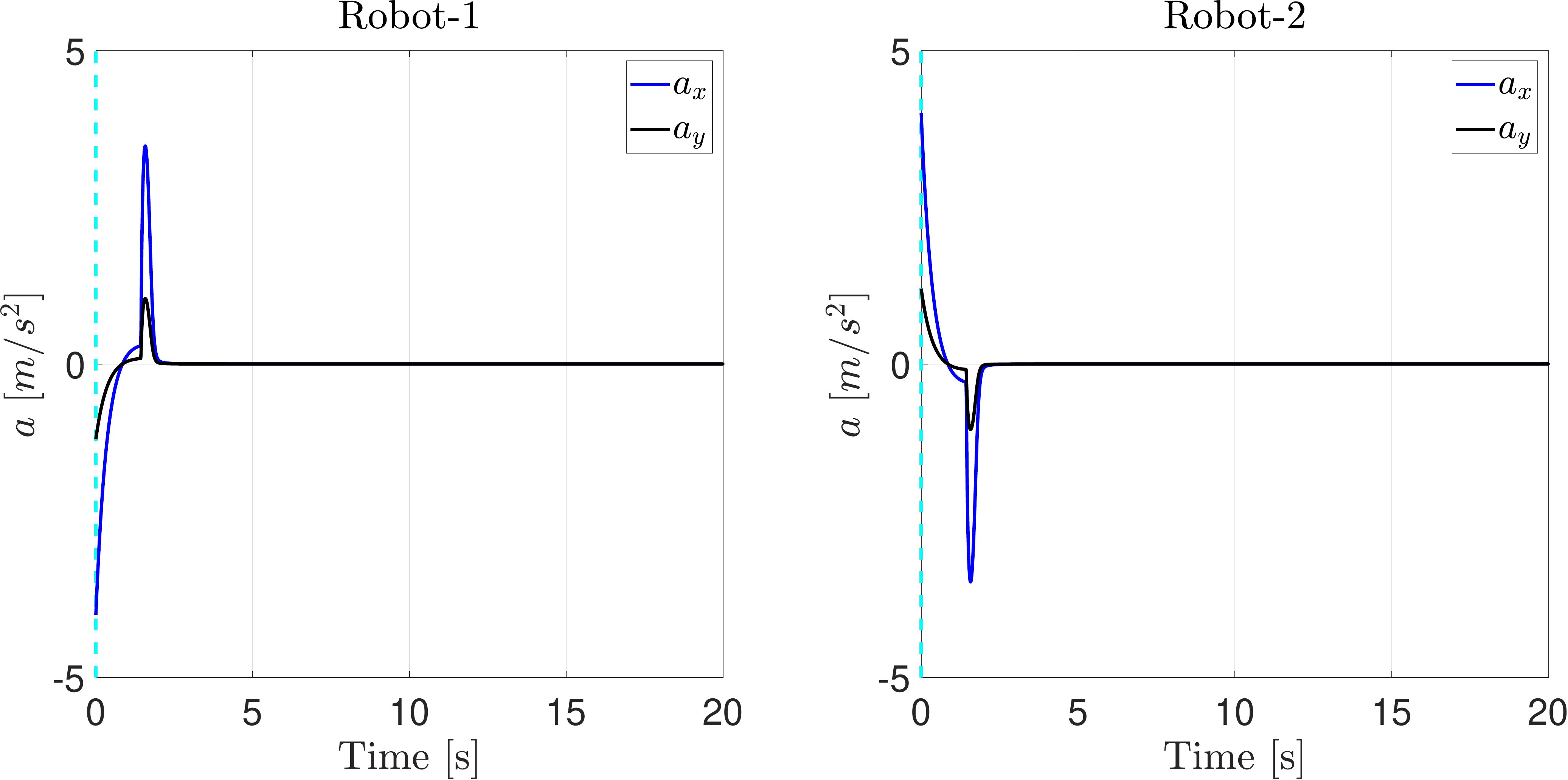}}
\caption{Positions and accelerations of robots falling in deadlock. Note  $\lim_{t\rightarrow \infty}p_{x,y}\neq p_{d_{x,y}}$ yet  $\lim_{t\rightarrow \infty}a_{x,y}=0$. Video: \url{https://tinyurl.com/y4ylzwh8}}
\label{fig:deadlock_traj}
\end{figure}

\subsection{KKT Conditions}
\label{KKT Conditions}
Recall  each robot $i$ computes a control by solving a local QP as in \cref{optimization_formulation_2}. 
Define $\boldsymbol{a}_j\coloneqq-\Delta\boldsymbol{p}_{ij}$ and $\hat{b}_j=\frac{\alpha_i}{\alpha_i+\alpha_j}b_{ij}$.We will drop subscript $i$ and implicitly assume  the QP is being solved for the ego robot. Hence, we rewrite \cref{optimization_formulation_2} as:

\begin{equation}
\label{P1}
	\begin{aligned}
		& \underset{\boldsymbol{u}}{\text{minimize}}
		& & \norm{\boldsymbol{u} - \hat{\boldsymbol{u}}}^2_2 \\
		& \text{subject to}
		& & \widetilde{A}\boldsymbol{u} \leq \widetilde{\boldsymbol{b}}
	\end{aligned}
\end{equation}
% \begin{align}
% \label{P1}
% & \underset{\boldsymbol{u}}{\text{minimize}}
% & & \norm{\boldsymbol{u} - \hat{\boldsymbol{u}}}^2_2 \nonumber \\
% & \text{subject to}
% & & \tilde{A}\boldsymbol{u} \leq \tilde{\boldsymbol{b}} 
% \end{align}
where $\tilde{A} \coloneqq (\boldsymbol{a}^T_1;\dots;\boldsymbol{a}^T_M;\boldsymbol{e}^T_1; \dots ;-\boldsymbol{e}^T_2)$ and $\tilde{\boldsymbol{b}}\coloneqq (\hat{b}_1;\dots;\hat{b}_M; \alpha;\dots; \alpha)$. % \begin{align}
% \mbox{where }&\tilde{A} =
%  \left[\begin{matrix}
% \boldsymbol{a}^T_1  \\ \vdots \\\boldsymbol{a}^T_M \\  \boldsymbol{e}^T_1 \\   \vdots \\  -\boldsymbol{e}^T_2 \\
% \end{matrix}\right] \mbox{  and   } 
% \tilde{\boldsymbol{b}} 
% = \left[\begin{matrix}
% \hat{b}_1  \\ \vdots \\\hat{b}_M \\  \alpha \\   \vdots \\  \alpha \\
% \end{matrix}\right]
% \end{align}
Let $\tilde{\boldsymbol{a}}_k$ denote the $k$'th row of $\tilde{A}$ and $\tilde{b}_k$ denote the $k$'th element of $\tilde{\boldsymbol{b}}$. The Lagrange dual function for \cref{P1} is
\begin{align}
L(\boldsymbol{u},\boldsymbol{\mu}) =  \norm{\boldsymbol{u} - \hat{\boldsymbol{u}}}^2_2  + \sum_{{k=1}}^{M+4}\mu_k(\boldsymbol{\tilde{a}}^T_k\boldsymbol{u}-\tilde{b}_k)
\end{align}
Let $(\boldsymbol{u}^*,\boldsymbol{\mu}^*)$ be the optimal primal-dual solution to \cref{P1}. The KKT conditions are
\begin{enumerate}
	\item Stationarity: $\nabla_{\boldsymbol{u}}L(\boldsymbol{u},\boldsymbol{\mu})\vert_{(\boldsymbol{u}^*,\boldsymbol{\mu}^*)} = \boldsymbol{0}$
	\begin{align}
	\label{stationarity1}
	&\implies \boldsymbol{u}^* = \hat{\boldsymbol{u}} - \frac{1}{2}\sum_{k=1}^{M+4}\mu^*_k\boldsymbol{\tilde{a}}^T_k.
	\end{align}
\item Primal Feasibility 
\begin{align}
\label{primal_feasibility1}
\tilde{\boldsymbol{a}}^T_k\boldsymbol{u}^* \leq \tilde{b}_k \mbox{  }	 \forall k \in \{1,2,\dots,M+4\}
\end{align}
\item Dual Feasibility 
\begin{align}
\label{dual_feasibility1}
{\mu^*_k} \geq 0 \mbox{  }	 \forall k \in \{1,2,\dots,M+4\}
\end{align}
\item Complementary Slackness 
\begin{align}
	\label{complimentarty slackness1}
	\mu^*_k \cdot (\boldsymbol{\tilde{a}}^T_k\boldsymbol{u}^* -\tilde{b}_k) = 0 
	 \mbox{   }\forall k \in \{1,2,\dots,M+4\}
\end{align}
\end{enumerate}
Define the set of active and inactive constraints as follows: 

\begin{align}
	\mathcal{A}(\boldsymbol{u}^*) = \{k \in \{1,2,\dots,M+4\} \mid \boldsymbol{\tilde{a}}^T_k\boldsymbol{u}^* = \tilde{b}_k \} \\
	\mathcal{IA}(\boldsymbol{u}^*) = \{k \in \{1,2,\dots,M+4\} \mid \boldsymbol{\tilde{a}}^T_k\boldsymbol{u}^* < \tilde{b}_k \} 
\end{align} \break
Using complementary slackness from \cref{complimentarty slackness1}, we deduce

\begin{align}
	\mu^*_k = 0 \mbox{ $\forall k $} \in \mathcal{IA}(\boldsymbol{u}^*)
\end{align}  
Therefore, we can restrict the sum in \cref{stationarity1} to only the set of active constraints 
\begin{align}
\label{kkt_general}
	\boldsymbol{u}^* = \hat{\boldsymbol{u}} - \frac{1}{2}\sum_{k \in \mathcal{A}(\boldsymbol{u}^*)}\mu^*_k\boldsymbol{\tilde{a}}^T_k
\end{align}

\subsection{KKT Conditions for the deadlock case}
\label{Nrobots_KKT_deadlock}
From Def. \ref{deadlock_definition}, we know  in deadlock, $\boldsymbol{u}^*=\boldsymbol{0}$, $\hat{\boldsymbol{u}}\neq 0$ and $\boldsymbol{v}=\boldsymbol{0}$. We conclude:
\begin{enumerate}
	\item In deadlock, $\boldsymbol{u}^* \neq \hat{\boldsymbol{u}}$ \textit{i.e.} the solution to the QP is not equal to the prescribed PD controller (which means  $\hat{\boldsymbol{u}}$ is infeasible in deadlock \textit{i.e. }$\boldsymbol{a}^T\hat{\boldsymbol{u}}\nleq \hat{b}	$).
	\item $\boldsymbol{u}^*=\boldsymbol{0} \implies \boldsymbol{u}^* \neq \pm \boldsymbol{\alpha}$. This implies  at least the last four constraints in $\tilde{A}\boldsymbol{u} \leq \tilde{b}$ are inactive \textit{i.e.} $\{M+1,M+2,M+3,M+4\} \in \mathcal{IA}(\boldsymbol{u}^*)$ in deadlock.
\end{enumerate}

Using these observations, we rewrite the KKT conditions for the deadlock case:
\begin{enumerate}
	\item Stationarity: $\nabla_{\boldsymbol{u}}L(\boldsymbol{u},\boldsymbol{\mu})\vert_{(\boldsymbol{0},\boldsymbol{\mu}^*)} = \boldsymbol{0}$
	\begin{align}
	\label{stationarity}
    \implies \hat{\boldsymbol{u}} = \frac{1}{2}\sum_{k\in \mathcal{A}(\boldsymbol{u}^*)}\mu^*_k \boldsymbol{\tilde{a}}^T_k
	\end{align}
	\item Primal Feasibility 
	\begin{align}
	\label{primal_feasibility}
    \tilde{b}_k \geq 0\mbox{  }	 \forall k \in \{1,2,\dots,M+4\} 
	\end{align}
	\item Dual Feasibility 
	\begin{align}
	\label{dual_feasibility}
{\mu^*_k} \geq 0 \mbox{  }	 \forall k \in \{1,2,\dots,M+4\}
	\end{align}
	\item Complementary Slackness 
	\begin{align}
	\label{complimentarty slackness}
	\mu^*_k \cdot (\tilde{\boldsymbol{a}}^T_k\boldsymbol{u}^*-\tilde{b}_k) = 0 \nonumber \\
	\implies \mu^*_k \cdot \tilde{b}_k = 0 \mbox{  $\forall$} j \in \{1,2,\dots,M+4\}
	\end{align}
\end{enumerate}
Based on these conditions, we will now motivate a set-theoretic interpretation of deadlock. Assume  the state of the ego robot is $\boldsymbol{z}=(\boldsymbol{p},\boldsymbol{v})$ and it has $M$ neighbors denoted as $\boldsymbol{Z}_{nb.}$. Define  $P \in \mathbb{R}^{2 \times 4}$ and $V \in \mathbb{R}^{2 \times 4}$ appropriately to extract the position and velocity components from $\boldsymbol{z}$ \textit{i.e.} $\boldsymbol{p}=P\boldsymbol{z}$ and $\boldsymbol{v}=V\boldsymbol{z}$. Finally, define $\mathcal{D}$ as:
\begin{align}
\label{deadlock_def1}
	\mathcal{D}(\boldsymbol{z}\mid \boldsymbol{Z}_{nb.})=\{\boldsymbol{z} \in \mathbb{R}^4 \mid {\boldsymbol{u}^*}(\boldsymbol{z}) = 0, \hat{\boldsymbol{u}}(\boldsymbol{z}) \neq 0,V\boldsymbol{z} = 0,  \mu_{k}^*(\boldsymbol{Z})> 0 \mbox{ $\forall$ } k\in \mathcal{A}(\boldsymbol{u^*})\}
\end{align}
\break The set $\mathcal{D}$ is defined as the set of all states of the ego robot which satisfy the criteria of being in deadlock. We have combined the conditions of deadlock into a set theoretic definition. Note  for each robot, its set of deadlock states depends on the states of its neighboring robots. This is because the Lagrange multipliers depend  on the states of all robots. The motivation behind stating this definition is to interpret deadlock as a bonafide set in the state space of the ego robot and derive a control strategy  makes the robot evade/exit this set. We now rewrite this definition in more easily interpretable conditions. From \cref{kkt_general} and \cref{stationarity}, note  
\begin{align}
\label{force_equate}
{\boldsymbol{u}}^*(\boldsymbol{z})=0 \iff \hat{\boldsymbol{u}}(\boldsymbol{z})=\frac{1}{2}\sum_{k\in \mathcal{A}({\boldsymbol{u}^*}(\boldsymbol{z}))}\mu^*_k \boldsymbol{\tilde{a}}_k
\end{align}
Since $\boldsymbol{\tilde{a}}_k = -\Delta \boldsymbol{p}_{ik}= -P(\boldsymbol{z}-\boldsymbol{z}_k)$, we rewrite \cref{force_equate} as:
\begin{flalign}
\label{force_equate3}
{\boldsymbol{u}}^*(\boldsymbol{z})=0 \iff \hat{\boldsymbol{u}}(\boldsymbol{z})=-\frac{1}{2}\sum_{k\in \mathcal{A}({\boldsymbol{u}^*}(\boldsymbol{z}))}\mu^*_k P(\boldsymbol{z}-\boldsymbol{z}_k)
\end{flalign}
\break We will use this condition to replace the ${\boldsymbol{u}}^*(\boldsymbol{z})=0$ criterion in the def. of $\mathcal{D}$ in (\ref{deadlock_def1}). Secondly, we  know  prescribed controller $\hat{\boldsymbol{u}}(\boldsymbol{z})$ is a PD controller. Define the goal state as $\boldsymbol{z}_d=(\boldsymbol{p}_d,\boldsymbol{0})$. Noting  $\hat{\boldsymbol{u}}(\boldsymbol{z}) \neq 0$ and $ \boldsymbol{v}= \mathbf{0}$,
\begin{align}
\label{cond2}
&\hat{\boldsymbol{u}}(\boldsymbol{z}) = -k_p(\boldsymbol{p}-\boldsymbol{p}_d) - k_v\boldsymbol{v} \neq  \mathbf{0} \iff P(\boldsymbol{z}-\boldsymbol{z}_d) \neq  \mathbf{0}
\end{align}
This criterion is restating  in deadlock the ego robot is not at its goal. The final condition is  the velocity of the ego robot is zero \textit{i.e.}
$\boldsymbol{v}=\boldsymbol{0} \iff V\boldsymbol{z}=\boldsymbol{0}$. Combining these conditions, we rewrite the definition of the deadlock from \cref{deadlock_def1} as follows:
\begin{align}
\label{deadlock_def2}
\mathcal{D}(\boldsymbol{z}\mid \boldsymbol{Z}_{nb.})=\{\boldsymbol{z} \in \mathbb{R}^4 \vert
&\hat{\boldsymbol{u}}(\boldsymbol{z})=-\frac{1}{2}\sum_{k\in \mathcal{A}({\boldsymbol{u}^*}(\boldsymbol{z}))}\mu^*_k P(\boldsymbol{z}-\boldsymbol{z}_k), P(\boldsymbol{z}-\boldsymbol{z}_d) \neq 0,  V\boldsymbol{z} = 0, \nonumber \\
& \mu_{k}^*> 0 \mbox{ $\forall$ } k\in \mathcal{A}(\boldsymbol{u^*}(\boldsymbol{z}))  \}
\end{align}
Building on the definition of one robot deadlock, we motivate \textit{system deadlock} to be the set of states where all robots are in deadlock and is defined as 
\begin{align}
\label{sysdeadlockdefN}
\mathcal{D}_{system}
=&\{(\boldsymbol{z}_1,\boldsymbol{z}_2,\cdots,\boldsymbol{z}_N)\in \mathbb{R}^{4N}\mid \boldsymbol{z}_i \in \mathcal{D}(\boldsymbol{z}_i\mid \boldsymbol{Z}^i_{nb.}) \mbox{  }\forall i \in \{1,2,\cdots,N\}) 
\end{align}
For the rest of the paper, we will focus our analysis on system deadlock. This is because the case where only a subset of robots are in deadlock can be decomposed into subproblems where a subset is in \textit{system deadlock} and the remaining robots free to move. The next section focuses on the geometric complexity analysis of \textit{system deadlock}. 
\subsection{Graph Enumeration based Complexity Analysis of Deadlock}
The Lagrange multipliers $\mu^*_k$  are in general, a nonlinear function of the state of robots $\boldsymbol{z}$. Their values depend on which constraints are active/inactive (an example calculation is shown in \cref{Lagrange Multipliers}). An active constraint will in-turn determine the set of possible geometric configurations  the robots can take when they are in deadlock (\cref{Analysis for two robots,Analysis for three robots}) and  this in turn will guide the design of our deadlock resolution algorithm (\cref{deadlock resolution}).  Therefore, we are interested in deriving all possible combinations of active/inactive constraints  the robots can assume once in deadlock. But first we derive upper and lower bounds for the number of valid configurations in \textit{system deadlock}. 

We can interpret an active collision avoidance constraint between robots $i$ and $j$ as an undirected edge between vertices $i$ and $j$ in a graph formed by $N$ labeled vertices, where each vertex represents a robot. The following property (which follows from symmetry) allows the edges to be undirected.
\begin{lemma}
    If robot $i$ and $j$ are both in deadlock and  $i's$  constraint with $j$ is active (inactive), then  $j's$  constraint with $i$ is also active (inactive).
\end{lemma}

\subsubsection{Upper Bound}
Given $N$ vertices, there are ${}^{N}C_{2}$ distinct pairs of edges possible. The overall system can have any subset of those edges. Since a set with ${}^{N}C_{2}$ members has $2^{{}^{N}C_{2}}$ subsets, we conclude  there are $2^{{}^{N}C_{2}}$ possible graphs. In other words, given $N$ robots, the number of configurations  are admissible in deadlock is $2^{{}^{N}C_{2}}$. However, this number is an upper bound because it includes cases where a given vertex can be disconnected from all other vertices, which is not valid in \textit{system deadlock} as shown next.

\subsubsection{Lower Bound}
We further impose the restriction  each vertex in the graph have at-least one edge \textit{i.e.} each robot have at-least one constraint active with some other robot. This is because if a robot has no active constraints \textit{i.e.} $\mu^*_k=0$ $\forall k$ then from \cref{force_equate}, we will get $\boldsymbol{u}^*(\boldsymbol{z}) = \hat{\boldsymbol{u}}(\boldsymbol{z}) = \boldsymbol{0}$ which would contradict the definition of deadlock for  robot and hence contradict \textit{system deadlock}. \js{From this observation, it follows  the set of graphs  are valid in \textit{system deadlock} is a superset of connected simple graphs  can be formed by $N$ labeled vertices. This is because there could be graphs  are not simply connected yet admissible in deadlock. While this argument is based on algebraic qualifiers resulting from the `edge' interpretation of collision avoidance constraints, it is possible  some simply connected graphs may not be geometrically feasible due to restrictions imposed by Euclidean geometry. Graphs  are (a) simply-connected (to enforce deadlock for each robot), (b) have $N$ labelled vertices (since each robot has an ID), (c) are embedded in $\mathbb{R}^2$ (since the robots/environment are planar), (d) have Euclidean distance between connected vertices equal to $D_s$, (e)  between unconnected vertices greater than $D_s$, and (f) have at-least one or two edges per vertex, necessarily represent admissible geometric configurations of robots in \textit{system deadlock}.
% \begin{enumerate}
%     \item (Algebraic) simply-connected (to enforce deadlock for each robot)
%     \item (Algebraic) has $N$ labelled vertices (each robot has an ID)
%     \item (Geometric) embedded in $\mathbb{R}^2$ (robots/environment are planar)
%     \item (Geometric) Euclidean distance between connected vertices equals  $D_s$
%     \item (Geometric) Euclidean distance between unconnected vertices exceeds $D_s$ 
% \end{enumerate}
The reason for qualifiers (d) and (e) is explained in the proof of \cref{touching}. (f) is needed because the decision variables in \cref{P1} are in $\mathbb{R}^2$, so there can be one or two active constraints (possibly more) per ego robot. The number of graphs meeting qualifiers (a) and (b) can be obtained using the following recurrence relation \cite{wilf2005generatingfunctionology}
\begin{align}
    d_N = 2^{{}^{N}C_{2}} - \frac{1}{N}\sum_{k=1}^{N-1}k \mbox{ }{{}^{N}C_{k}}2^{{}^{N-k}C_{2}}d_k
\end{align}
For $N=\{1,\mbox{ }2,\mbox{ }3,\mbox{ }4\}$, this number is $\{1,\mbox{ }1,\mbox{ }4,\mbox{ }38\}$. The number of graphs meeting qualifiers (a), (c) and (d) can be obtained by calculating the number of connected matchstick graphs on $N$ nodes \cite{Weisstein}. The number of graphs meeting  (b) and (d) was obtained in \cite{alon2014two} and is exponential in $N^2$ (for unit distance graphs). A lower bound for graphs satisfying all qualifiers (a)-(f) can be shown to be $0.5(N+1)(N-1)!$ as follows ($N\geq 3$). Consider a cyclic graph whose each node is the vertex of an $N$ regular polygon with side $D_s$. Such a graph necessarily satisfies (a)-(f). Re-arrangements of its vertices gives rise to $0.5(N-1)!$ graphs. Likewise, a graph with nodes along an open chain also satisfies (a)-(f), and gives $0.5N!$ rearrangements. Thus, the total is $0.5(N-1)!+0.5N!=0.5(N+1)(N-1)!$ It is well known  factorial overtakes exponential, thus highlighting the increase in the number of geometric configurations. Our MATLAB simulations show  the exact number of configurations for $N=\{1,\mbox{ }2,\mbox{ }3,\mbox{ }4\}$ are $\{1,\mbox{ }1,\mbox{ }4,\mbox{ }18\}$ whereas our bound gives $\{1,\mbox{ }1,\mbox{ }4,\mbox{ }15\}$.
% is $N!\hspace{0.1cm} c^N$ where $c$ is a constant and $N!$ comes from permuting labels \cite{complexity}. Exact calculation is currently under investigation. Our MATLAB simulations show  the exact number of geometric configurations admissible in deadlock for $N=\{1,\mbox{ }2,\mbox{ }3,\mbox{ }4\}$ are $\{1,\mbox{ }1,\mbox{ }4,\mbox{ }40\}$ respectively (which is close to the number meeting qualifiers (a) and (b) \textit{i.e.} $\{1,\mbox{ }1,\mbox{ }4,\mbox{ }38\}$). (\textit{See supplementary for all valid deadlocks for four robots}). 
This simulation demonstrates the explosion in the number of possible geometric configurations  are admissible in \textit{system deadlock} with increasing number of robots. Therefore for further analysis, we will restrict to the case of two and three robots. 
}
% Therefore, a lower bound on the number of possible geometric configurations in deadlock can be obtained by calculating the number of connected simple graphs using $N$ labeled vertices. For a given $N$, this lower bound $d_N$ satisfies the following recurrence relation (see \cite{wilf2005generatingfunctionology} (3.10.2))
% \begin{align}
%     d_N = 2^{{}^{N}C_{2}} - \frac{1}{N}\sum_{k=1}^{N-1}k \mbox{ }{{}^{N}C_{k}}2^{{}^{N-k}C_{2}}d_k
% \end{align}
% For $N=\{1,\mbox{ }2,\mbox{ }3,\mbox{ }4,\mbox{ }5\}$, the lower bounds are calculated to be $\{1,\mbox{ }1,\mbox{ }4,\mbox{ }38,\mbox{ }728\}$ respectively whereas the actual number of cases valid in deadlock are $\{1,\mbox{ }1,\mbox{ }4,\mbox{ }41,\mbox{ }768\}$ respectively. This lower bound demonstrates the explosion in the number of possible geometric configurations  are admissible in \textit{system deadlock} with increasing number of robots. Therefore for further analysis, we will restrict to the case of two and three robots.  (\textit{See supplementary material for all valid deadlocks for four robots}).

\section{Two-Robot Deadlock}

\label{Analysis for two robots}
In \cref{Analysis for deadlock}, we proposed a set-theoretic definition of deadlock for a specific robot in an $N$ robot system. In this section, we will refine the KKT conditions derived in \cref{Nrobots_KKT_deadlock} for the case of two robots in the system. This setting reveals several important underlying characteristics of the system  are extendable to the $N$ robot case, as will be shown for $N=3$. One key feature of a two-robot system is  a single robot by itself cannot be in deadlock \textit{i.e.} either both robots are in deadlock  or neither. This is because the sole collision avoidance constraint is symmetric due to Lemma 1. Hence, a two-robot system can only exhibit \textit{system deadlock}. Additionally since the ego robot avoids collision only with the one other robot, there is no sum in \cref{force_equate} \textit{i.e.}
\begin{align}
\label{force_equate2robot}
{\boldsymbol{u}}^*(\boldsymbol{z})=0 \iff \hat{\boldsymbol{u}}(\boldsymbol{z})=\frac{1}{2}\mu^*\boldsymbol{a}
\end{align}
The left hand side of this equation is $\hat{\boldsymbol{u}}(\boldsymbol{z}) = -k_p(\boldsymbol{p}_{ego}-\boldsymbol{p}_d)$. The right hand side is $\frac{1}{2}\mu^* \boldsymbol{a} = -\frac{1}{2}\mu^*(\boldsymbol{p}_{ego}-\boldsymbol{p}_{neighbor})$. Writing this another way, we have $-k_p(\boldsymbol{p}_{ego}-\boldsymbol{p}_d) + \frac{1}{2}\mu^*(\boldsymbol{p}_{ego}-\boldsymbol{p}_{neighbor})=\boldsymbol{0}$. The first term as represents an attractive force  pulls the ego robot towards goal $\boldsymbol{p}_d$. Since $\mu^*>0$,  the second term represents a repulsive force pushing the ego robot away from its neighbor. Thus, \textit{system deadlock} occurs when the net force due to attraction and repulsion on each robot vanishes (see Fig. \ref{fig:D4.png}).
We now define the \textit{system deadlock} set $\mathcal{D}_{system}$ using \cref{deadlock_def2,sysdeadlockdefN}:
\begin{align}
\label{sysdeadlockdef}
\mathcal{D}_{system}
=&\{(\boldsymbol{z}_1,\boldsymbol{z}_2)\in \mathbb{R}^8\mid \boldsymbol{\hat{u}}_1=\frac{1}{2}\mu^*_1\boldsymbol{a}_1,\mbox{ }\boldsymbol{\hat{u}}_2=\frac{1}{2}\mu^*_2\boldsymbol{a}_2,\mbox{ }\mu^*_1> 0 ,\mbox{ } \mu^*_2> 0,\mbox{ } \nonumber \\ &(P(\boldsymbol{z}_1-\boldsymbol{z}_{d_1}),P(\boldsymbol{z}_{2}-\boldsymbol{z}_{d_2}))\neq (\boldsymbol{0},\boldsymbol{0}),\mbox{ } (V\boldsymbol{z}_1,V\boldsymbol{z}_2)=(\boldsymbol{0},\boldsymbol{0})\}.
\end{align}
where $\boldsymbol{a}_1 =-\boldsymbol{a}_2 = -(\boldsymbol{p}_1-\boldsymbol{p}_2)$. Next, we derive analytical expressions for the Lagrange multipliers $\mu^*_1, \mu^*_2$. Depending on whether the collision avoidance constraint is active/inactive at the optimum, there are two cases: \\ 
\textbf{Case 1}: The constraint $\boldsymbol{a}^T\boldsymbol{u} \leq \hat{b} $ is active at $\boldsymbol{u}=\boldsymbol{u}^*$ \textit{i.e.} $\boldsymbol{a}^T\boldsymbol{u}^*= \hat{b}$
	\begin{align}
	\label{Lagrange Multipliers}
		&\implies \boldsymbol{a}^T\bigg(\hat{\boldsymbol{u}} - \frac{1}{2}\mu^*\boldsymbol{a}\bigg)= \hat{b} \nonumber \\
		&\implies\mu^* = 2\frac{\boldsymbol{a}^T\hat{\boldsymbol{u}} - \hat{b}}{\norm{ \boldsymbol{a}}^2_2}
	\end{align}
\textbf{Case 2}: The constraint $\boldsymbol{a}^T\boldsymbol{u} \leq \hat{b} $ is inactive at $\boldsymbol{u}=\boldsymbol{u}^*$ \textit{i.e.} $\boldsymbol{a}^T\boldsymbol{u}^* < \hat{b} $. From complementary slackness, it follows  $\mu^* = 0$ and hence $\boldsymbol{u}^* = \hat{\boldsymbol{u}}$. However, this contradicts the definition of deadlock. Hence, case 2 can never arise in deadlock. 
\subsection{Characteristics of two-robot deadlock}
We now analyze qualitative properties of the system deadlock set towards synthesizing a controller  will enable the robots to exit this set. We will show  when deadlock occurs, (1) the two robots are separated by the safety distance, (2) deadlock set is non-empty and (3) bounded and of measure zero.
\begin{theorem}[Safety Margin Apart]
\label{touching}
In deadlock, the two robots are separated by the safety distance and the robots are on the verge of violating safety (see Fig. \ref{fig:deadlock_cartoon}, \ref{fig:D4.png})
\end{theorem}
\begin{proof}
In \cref{Lagrange Multipliers}, we proved  the collision avoidance constraint is active in deadlock. Since both robots are in deadlock, we know  both of their collision avoidance constraints are active \textit{i.e.} $
\boldsymbol{a}^T_1\boldsymbol{u}_1^*= \hat{b}_{12} \mbox{  ,  } \boldsymbol{a}^T_2\boldsymbol{u}_2^*= \hat{b}_{21} \nonumber 
$ and $\boldsymbol{u}_1^*=\boldsymbol{0}$ and  $\boldsymbol{u}_2^*=\boldsymbol{0}$. This implies $\hat{b}_{12}=\hat{b}_{21}=0$. Using \cref{safetyconstraint,decentralizedconstraints} and  in deadlock, $(\boldsymbol{v}_1,\boldsymbol{v}_2)=(\boldsymbol{0},\boldsymbol{0})$ we get 
\begin{align}
\label{actuallyneededinboundary}
\hat{b}_{12}=\frac{\alpha_1}{\alpha_1+\alpha_2}\norm{\Delta \boldsymbol{p}_{12}}h^3_{12}=0 \implies h_{12}=0 
\end{align}
Recall $h_{12}$ from \cref{hdef} and using  $(\boldsymbol{v}_1,\boldsymbol{v}_2)=(\boldsymbol{0},\boldsymbol{0})$, we get 
\begin{align}
\label{neededinboundary}
h_{12}(\boldsymbol{z}_1,\boldsymbol{z}_2) = \sqrt{2(\alpha_1 + \alpha_2)(\norm{\Delta \boldsymbol{p}_{12}}-D_{s})} 
\end{align} 
Therefore, $h_{12}=0 \iff  \norm{\Delta \boldsymbol{p}_{12}}=D_{s}$. Assuming QP is feasible, we disregard $\norm{\Delta \boldsymbol{p}_{12}}=0$. Therefore, $\norm{\Delta \boldsymbol{p}_{12}}=D_{s}$. Additionally, recalling the definition from $\partial \mathcal{C}$ from \cref{safetysetdef} we deduce , in deadlock,  $(\boldsymbol{z}_1,\boldsymbol{z}_2) \in \partial \mathcal{C}$ \textit{i.e.} $\mathcal{D}_{system} \subset \partial \mathcal{C}$. 
\end{proof}
% \subsection{\mbox{$\mathcal{D}_{system}$ is located on the boundary of the safe set}}
This result confirms our intuition, because if the robots are separated by more than the safety distance, then they will have wiggle room to move because they are not at their goals and $\hat{\boldsymbol{u}}\neq \boldsymbol{0}$. However, the ability to move, albeit with small velocity would contradict the definition of deadlock. We now propose a family of states  are always in the system deadlock set $\mathcal{D}_{system}$.
\begin{theorem}[$\mathcal{D}_{system}$ is Non-Empty] 
\label{nonempty}
$\forall\mbox{ } k_p,k_v,D_s>0,\exists$ a family of states $(\boldsymbol{z}^*_1,\boldsymbol{z}^*_2) \in \mathcal{D}_{system}$. These states are such  the robots and their goals are all collinear.
\end{theorem}
\begin{proof}
To prove this theorem, we propose a set of candidate states $(\boldsymbol{z}^*_1,\boldsymbol{z}^*_2)$ and show  they satisfy the definition of deadlock \cref{sysdeadlockdef}. See Fig. \ref{fig:D4.png} for an illustration of geometric quantities referred to in this proof.

Let $\boldsymbol{p}^*_1=\alpha \boldsymbol{p}_{d_1} + (1-\alpha) \boldsymbol{p}_{d_2}$ and $\boldsymbol{p}^*_2= \boldsymbol{p}^*_1 - D_s\hat{e}_{\beta}$ where $\beta = \mbox{tan}^{-1}(\frac{y_{d_2}-y_{d_1}}{x_{d_2}-x_{d_1}})$ and $\alpha \in (0,1)$. Note  $\boldsymbol{p}^*_1,\boldsymbol{p}^*_2, \boldsymbol{p}_{d_1}, \boldsymbol{p}_{d_2}$ are collinear by construction. Let $\boldsymbol{z}^*_1=(\boldsymbol{p}^*_1,\boldsymbol{0})$ and $\boldsymbol{z}^*_2=(\boldsymbol{p}^*_2,\boldsymbol{0})$. Then we will show  $\boldsymbol{Z^*}=(\boldsymbol{z}^*_1,\boldsymbol{z}^*_2) \in \mathcal{D}_{system}$. 
Note 
\begin{align}
\centering
\label{cont}
\boldsymbol{a}_1&=-(\boldsymbol{p}^*_1-\boldsymbol{p}^*_2) =-D_{s}\hat{e}_{\beta} \nonumber \\
\hat{\boldsymbol{u}}_1&=-k_p(\boldsymbol{p}^*_1-\boldsymbol{p}_{d_1}) 
\end{align}
From definition, $\hat{e}_{\beta} 
    =\frac{1}{D_G}(x_{d_2}-x_{d_1},y_{d_2}-y_{d_1})$ where $D_G = \norm{\boldsymbol{p}_{d_2}-\boldsymbol{p}_{d_1}}$ is the distance between the goals. Therefore, we have 
\begin{align}
\label{etheta}
\boldsymbol{p}^*_1-\boldsymbol{p}_{d_1} &= 
% \alpha \boldsymbol{p}_{d_1} + (1-\alpha) \boldsymbol{p}_{d_2} - \boldsymbol{p}_{d_1} \nonumber \\
-(1-\alpha) \boldsymbol{p}_{d_1} + (1-\alpha) \boldsymbol{p}_{d_2} \nonumber \\
% &= (1-\alpha)(x_{d_2}-x_{d_1},y_{d_2}-y_{d_1}) \nonumber \\
% &= (1-\alpha)D_{G}\bigg(\frac{x_{d_2}-x_{d_1}}{D_G},\frac{y_{d_2}-y_{d_1}}{D_G}\bigg) \nonumber \\
&= (1-\alpha)D_{G}\hat{e}_{\beta}
\end{align}
Substituting \cref{etheta} in \cref{cont} gives 
\begin{align}
\label{lhs}
\hat{\boldsymbol{u}}_1=-k_p(1-\alpha)D_{G}\hat{e}_{\beta}
\end{align}
From \cref{cont} and \cref{lhs}, we deduce  Lagrange multiplier $\mu_1 $
\begin{align}
\label{lmult}
 \mu_1 &= 2\frac{\boldsymbol{a}^T_1\hat{\boldsymbol{u}}_1}{\norm{ \boldsymbol{a}_1}^2_2} = 2k_p(1-\alpha)\frac{D_G}{D_s} >0\mbox{  }\forall \alpha \in (0,1) \nonumber \\
\implies \frac{1}{2}\mu_1 \boldsymbol{a}_1 &= -\frac{1}{2}2k_p(1-\alpha)\frac{D_G}{D_s} D_{s}\hat{e}_{\beta}  =\hat{\boldsymbol{u}}_1
\end{align}
Hence, in \cref{lmult}, we have shown  $\hat{\boldsymbol{u}}_1=\frac{1}{2}\mu_1 \boldsymbol{a}_1$ which is one condition in the definition of the deadlock set. Similarly, we can show  $\hat{\boldsymbol{u}}_2=\frac{1}{2}\mu_2 \boldsymbol{a}_2$.  Also note  in \cref{lmult} we have shown  the Lagrange multiplier $\mu_1$ is positive, which is another condition in \cref{sysdeadlockdef}. We can similarly show  $\mu_2>0$. Finally, note  for our choice of states, $\boldsymbol{v}^*_1=\boldsymbol{v}^*_2=\boldsymbol{0}$ and we have restricted $\alpha \in (0,1)$ so we can ensure  $\boldsymbol{p}^*_i\neq\boldsymbol{p}_{d_i}$ Hence, the proposed states are always in deadlock. 
\end{proof}
\begin{theorem}[$\mathcal{D}_{system}$ is bounded]
The \textit{system deadlock} set is bounded and measure zero.
\end{theorem}
\begin{proof}
Following the definition of $\boldsymbol{p}^*_{1}$ and $\boldsymbol{p}^*_{2}$  from \cref{touching} and \cref{nonempty}, we can show  when two robots are in deadlock, their positions satisfy
\begin{align}
\norm{\big(\boldsymbol{p}_{1} -\boldsymbol{p}_{d_1}\big)} +  \norm{\big(\boldsymbol{p}_{2} -\boldsymbol{p}_{d_2}\big)} = D_s+ D_G \nonumber
\end{align}
This can be verified by straightforward substitution.  From this constraint it is evident,  the deadlock set is not ``large", bounded and of measure zero.  is why, random perturbations are one feasible way to resolve deadlock. 
\end{proof}

 \begin{figure}[t]
\centering      
\setlength{\belowcaptionskip}{-14pt}
\subfigure[Two Robot Equilibrium]{\label{fig:D4.png}\includegraphics[trim={0.0cm 15.8cm 11.94cm 2.7cm},clip,width=.35\columnwidth]{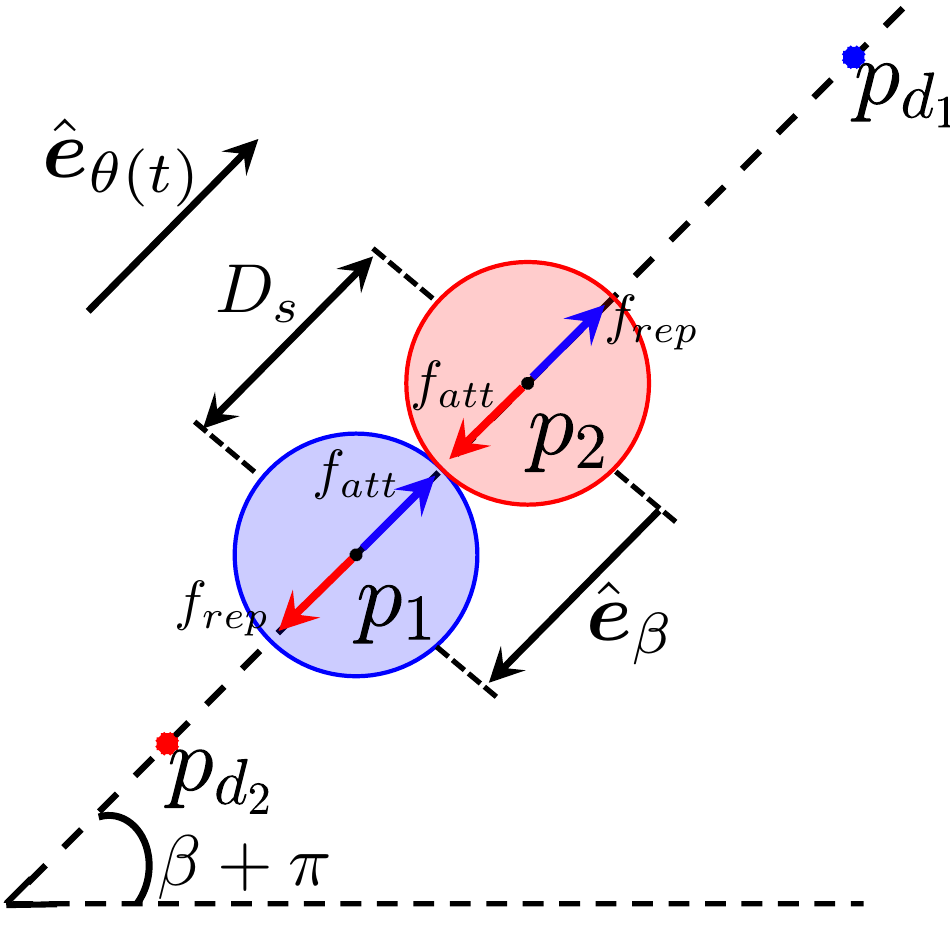}}
\subfigure[Three Robot Equilibrium]{\label{fig:D4.jpeg}\includegraphics[trim={11.2cm 15.4cm 0.0cm 2.7cm},clip,width=.35\columnwidth]{A.pdf}}
\caption{Force Equilibrium in Deadlock}
\label{equib._deadlock_two_three_robots}
\end{figure}
\section{Three Robot Deadlock}

\label{Analysis for three robots}
Following the ideas developed for two robot deadlock, we now describe the three robot case. We will demonstrate  properties such as robots being on the verge of safety violation (\cref{touching3}) and non-emptiness (\cref{nonempty3}) are retained in this case as well. We are interested in analyzing \textit{system deadlock}, which occurs when $\boldsymbol{u}^*_{i} =  \boldsymbol{0}$, $\boldsymbol{v}_{i} =  \boldsymbol{0}$ and $\hat{\boldsymbol{u}}_i \neq \boldsymbol{0}$ $\forall i \in \{1,2,3\}$. Since we are studying \textit{system deadlock}, each robot will have at-least one active collision avoidance constraint  (each robot has two constraints in total). Note  the \textit{system deadlock} set $\mathcal{D}_{system}$ for three robots is defined analogously to \cref{sysdeadlockdef}.
\begin{theorem}[Safety Margin Apart]
\label{touching3}
In system deadlock, either all three robots are separated by the safety margin or exactly two pairs of robots are separated by the safety margin.
\end{theorem}
\begin{proof}
The proof is kept brief because it is similar to the proof of \cref{touching}. Based on the number of constraints  are allowed to be active per robot, all geometric configurations can be clubbed in two categories :  \\
\textbf{Category A -}  This arises when all collision avoidance constraints of each robot are active \textit{i.e.} $\boldsymbol{a}^T_{ij}\boldsymbol{u}^*_i=\hat{b}_{ij}=0 \iff \norm{\Delta \boldsymbol{p}_{ij}}=D_s \mbox{  }\forall j\in \{1,2,3\}\backslash i$ $\forall i \in \{1,2,3\} $. As a result, each robot is separated by $D_s$ from every other robot (\cref{fig:equilateral}).  
\\
\textbf{Category B -} This arises when there is exactly one robot with both its constraints active (robot $i$ in \cref{fig:nonequilateral}), and the remaining two robots ($j$ and $k$) have exactly one constraint active each. Hence, robot $i$ is separated by $D_s$ from the other two. After relabeling of indices, category B results in three rearrangements. 
\end{proof}

\begin{theorem}[Non-emptiness] 
\label{nonempty3}
$\forall\mbox{ } k_p,k_v,D_s,R$ $>0$ and $\boldsymbol{p}_{d_i}=R\hat{\boldsymbol{e}}_{2\pi (i-1)/3}$ where $i=\{1,2,3\}$, $\exists$ $(\boldsymbol{z}^*_1,\boldsymbol{z}^*_2,\boldsymbol{z}^*_3) \in  \mathcal{D}_{system}$ where $\boldsymbol{z}^*_i=(\boldsymbol{p}^*_i,\boldsymbol{0})$ and $\boldsymbol{p}^*_i$ is proposed as follows:  \\
Category A: $\boldsymbol{p}^*_i=\frac{D_s}{\sqrt{3}}\hat{\boldsymbol{e}}_{\frac{2\pi (i-1)}{3} + \pi}  $ where $i=\{1,2,3\}$ \\
Category B: $\boldsymbol{p}^*_1 = D_s\hat{\boldsymbol{e}}_{\pi}$,
    $\boldsymbol{p}^*_2=\boldsymbol{0}$,   $\boldsymbol{p}^*_3=D_s\hat{\boldsymbol{e}}_{\frac{\pi}{3}}$ if robot $2$ has both constraints active.
%     \item  $\boldsymbol{p}^*_1 = \frac{-1}{2\sin{(\alpha-\theta)}} \begin{bmatrix}
%     2D_s\cos{\theta}\sin{(\alpha-\theta)}+2R\cos{\theta}\sin{(\alpha-\frac{\pi}{3})+2R\cos{\alpha}\sin{\theta}}\\
%     \sin{\theta}\big(3R\sin{\alpha}+2D_s\sin{(\alpha-\theta)}-\sqrt{3}R\cos{\alpha}\big)
%   \end{bmatrix}$
%     $\boldsymbol{p}^*_2=\boldsymbol{p}^*_1+D_s\hat{\boldsymbol{e}}_{\theta}$,   $\boldsymbol{p}^*_3=\boldsymbol{p}^*_2+D_s\hat{\boldsymbol{e}}_{\alpha}$, then $(\boldsymbol{z}^*_1,\boldsymbol{z}^*_2,\boldsymbol{z}^*_3) \in \mathcal{D}_{system}$ where $\boldsymbol{z}^*_i=(\boldsymbol{p}^*_i,\boldsymbol{0})$
\end{theorem}
 \begin{figure}[t]
	\centering      
	\subfigure[Category A ]{\label{fig:equilateral}\includegraphics[width=0.240\linewidth]{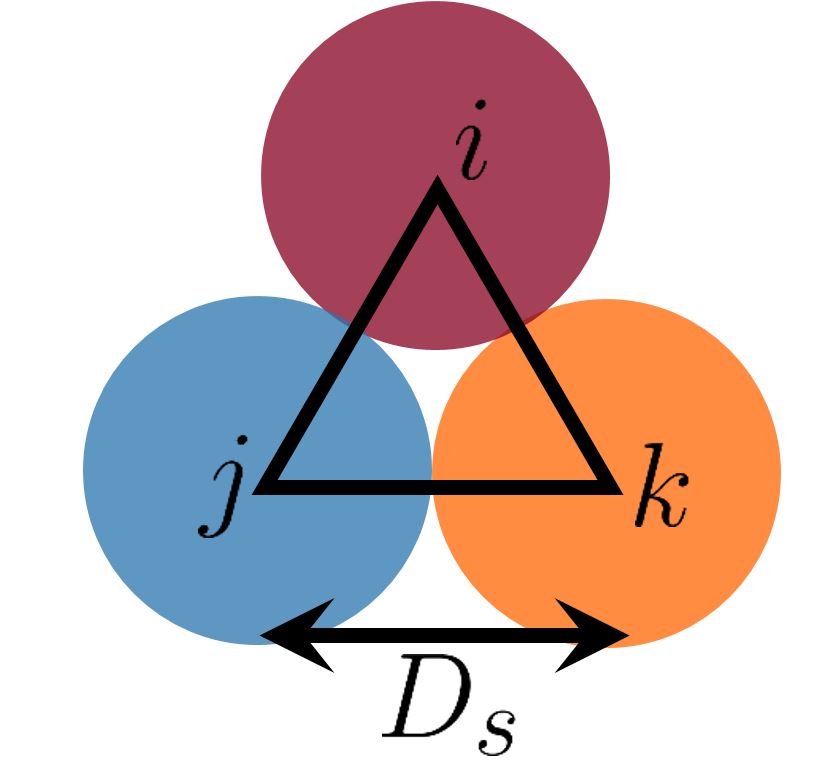}}
	\subfigure[Category B]{\label{fig:nonequilateral}\includegraphics[width=0.240\linewidth]{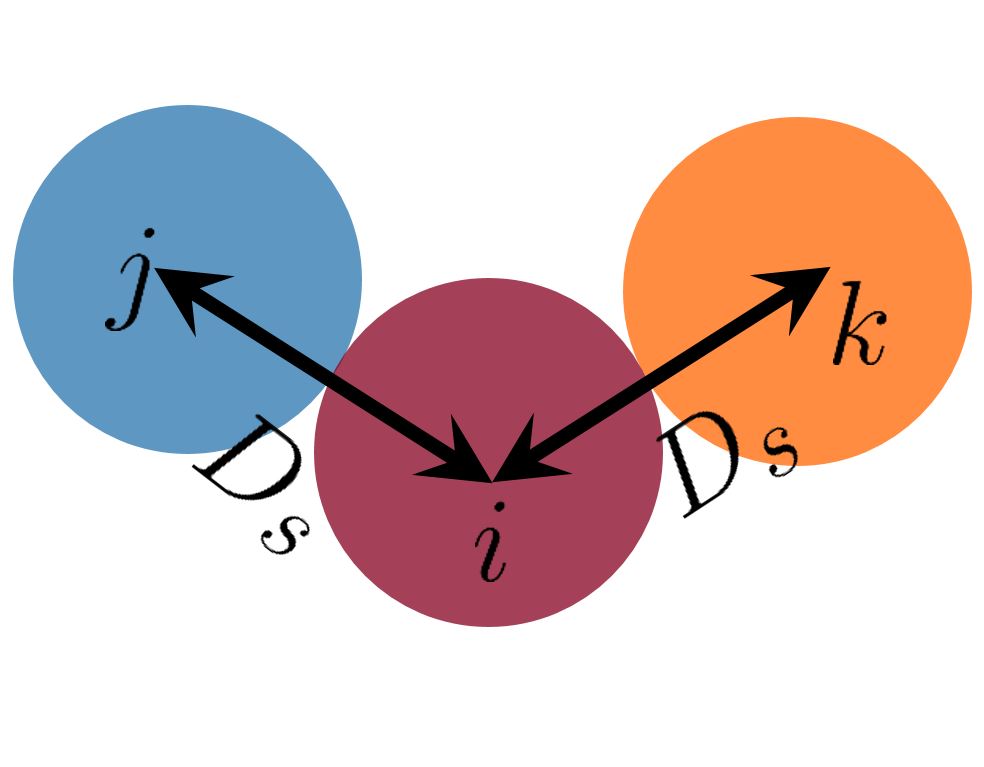}}
	\caption{Geometric configurations in system deadlock of three robots}
	\label{deadlock_three_robots}
\end{figure}
\begin{proof}
This proof is similar to the proof of \cref{nonempty} so it is skipped. Some remarks:
\begin{enumerate}
    \item In the statement of this theorem, we have predefined the desired goal positions unlike the statement of \cref{nonempty}. The candidate positions of the robots  we propose are in $\mathcal{D}_{system}$ are valid with respect to these given goals. We have derived a similar non-emptiness result for arbitrary goals but are not including it here for the sake of brevity.
    \item For category B, we proposed one set of positions  is valid in deadlock, however there is continuous family of positions  can be valid in category B. The representation of this family can be found in the appendix in \cref{nonempty32}. 
\end{enumerate}
\end{proof}
\section{Deadlock Resolution}
\label{deadlock resolution}
We now use the properties of geometric configurations derived in \cref{Analysis for two robots} and \cref{Analysis for three robots}  to synthesize a strategy  (1) gets the robots out of deadlock, (2) ensures their safety and (3) makes them converge to their goals.  One approach to achieve these objectives is to detect the incidence of deadlock while the  CBF-QP controller is running on the robots and once detected, any small non-zero perturbation to the control will instantaneously give a non-zero velocity to the robots. Thereafter, CBF-QPs can take charge again and we can hope  using this controller the system state will come out of deadlock at-least for a short time. This has two limitations however; firstly, since it was the CBF-QP controller  led to deadlock, there is no guarantee  the system will not fall back in deadlock again. Secondly, perturbations can violate safety and even lead to degraded performance. Therefore, we propose a controller which ensures  goal stabilization, safety and deadlock resolution are met with guarantees. We demonstrate this algorithm for the two and three robot cases. Extension to $N\geq 4$ is left for future work since $N\geq 4$ admits a large number of  geometric configurations  are valid in deadlock. Refer to  Fig. \ref{deadlock_res_two_three_robots} for a schematic of our approach. This algorithm is described here:
% \begin{figure}[t]
% \centering      
% \includegraphics[trim={0.0cm 24.3cm 0.0cm 0cm},clip,height=2cm,width=\columnwidth]{Doc1.pdf}
%  \setlength{\belowcaptionskip}{-20pt}
% \caption{Deadlock Resolution Algorithm Schematic}
% \label{deadlock_res_two_three_robots}
% \end{figure}
\begin{figure}[t]
\centering      
\subfigure[Deadlock resolution for two robots]{\label{fig:drestwo}\includegraphics[trim={0.0cm 24.2cm 0.0cm 0cm},clip,height=2cm,width=\columnwidth]{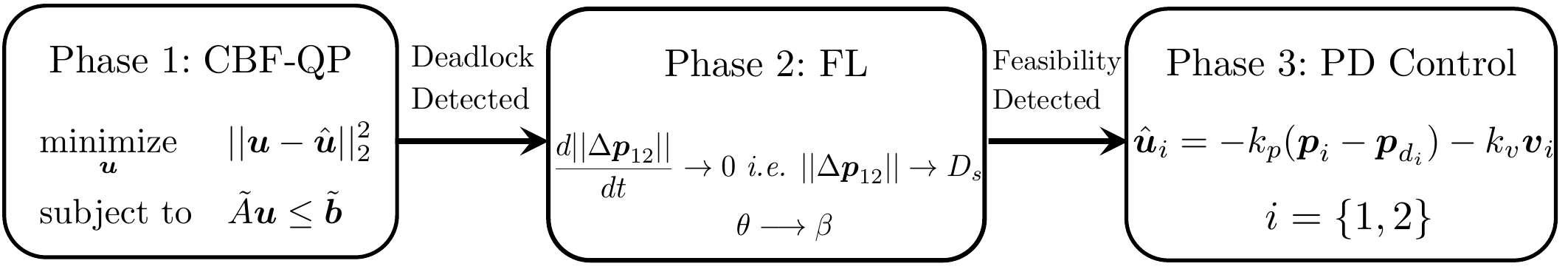}}
\subfigure[Deadlock resolution for three robots]{\label{fig:dresthree}\includegraphics[trim={0.0cm 24.2cm 0.0cm 0cm},clip,height=2cm,width=\columnwidth]{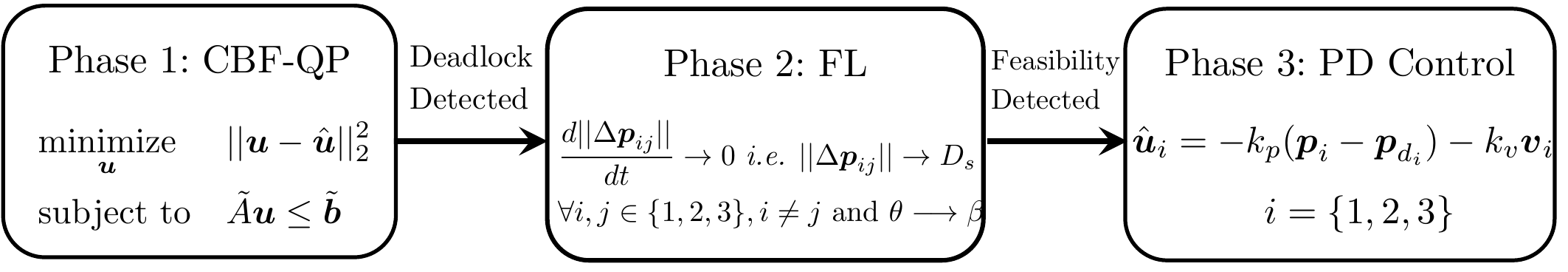}}
\caption{Deadlock Resolution Algorithm Schematic}
\label{deadlock_res_two_three_robots}
\end{figure}
\begin{enumerate}
    \item The algorithm starts by executing controls derived from CBF-QP in Phase 1. This ensures movement of robots to the goals and safety by construction. To detect the incidence of deadlock, we continuously compare $\norm{\boldsymbol{u}^*},\norm{\boldsymbol{v}}, \norm{\boldsymbol{p}-\boldsymbol{p}_{d}}$ against small thresholds. If satisfied, we switch control to phase 2, otherwise, phase 1 continues to operate.
    \item In this phase, we rotate the robots around each other to swap positions while maintaining the safe distance. 
    \begin{enumerate}
        \item For the two-robot case, we calculate $\boldsymbol{u}^1_{fl}(t)$ and $\boldsymbol{u}^2_{fl}(t)$ using feedback linearization (\textit{see  appendix material}) to ensure  $\norm{\Delta \boldsymbol{p}_{12}}=D_s$ and rotation ($\ddot{\theta}=-k_p(\theta-\beta)-k_v\dot{\theta} \implies \Delta \boldsymbol{p}^T_{12}\Delta \boldsymbol{v}_{12}=0$) (See Fig. \ref{fig:D4.jpeg} for $\theta,\beta$). This rotation and distance invariance guarantees safety \textit{i.e.} $\ h_{12}=0$. Adding an extra constraint $\boldsymbol{u}^1_{fl}+\boldsymbol{u}^2_{fl}=\boldsymbol{0}$ still ensures  the problem is well posed and additionally makes the centroid static. 
        \item For the three-robot case, we calculate $\boldsymbol{u}^1_{fl}(t),\boldsymbol{u}^3_{fl}(t),\boldsymbol{u}^3_{fl}(t)$ to ensure  $\norm{\Delta \boldsymbol{p}_{12}}=\norm{\Delta \boldsymbol{p}_{23}}=\norm{\Delta \boldsymbol{p}_{31}}=D_s$ and rotation ($\ddot{\theta}=-k_p(\theta-\beta)-k_v\dot{\theta} \implies \Delta \boldsymbol{p}^T_{12}\Delta \boldsymbol{v}_{12}=0$) (See Fig. \ref{fig:D4.jpeg} for $\theta,\beta$). This guarantees safety \textit{i.e.} $ h_{12}=h_{23}=h_{31}=0$. Similarly, we impose $\boldsymbol{u}^1_{fl}+\boldsymbol{u}^2_{fl}+\boldsymbol{u}^3_{fl}=\boldsymbol{0}$ to make the centroid static.
    \end{enumerate}
    \item Once the robots swap their positions, their new positions will ensure  prescribed PD controllers will be feasible in the future. Thus, after convergence of Phase 2 (which happens in finite time), control switches to Phase 3, which simply uses the prescribed PD controllers. This phase guarantees  the distance between robots is non-decreasing and safety is maintained as we prove in \cref{finalthm}.
\end{enumerate}
Fig. \ref{fig:simexpres} shows simulation and experimental results from running this strategy on two (\ref{fig:cc}) and three (Fig. \ref{fig:dd}) robots. Experiments were conducted using Khepera 4 nonholonomic robots (\ref{fig:ee}).  Note  for nonholnomic robots, we noticed from experiments and simulations  deadlock only occurs if the body frames of both robots are perfectly aligned with one another at $t=0$. Since this alignment is difficult to establish in experiments, we simulated a virtual deadlock at $t=0$ \textit{i.e.} assumed  the initial position of robots are ones  are in deadlock.

We next prove  this strategy ensures resolution of deadlock and convergence of robots to their goals. The proof of this theorem will exploit the geometric properties of deadlock we derived in \cref{touching} and \cref{nonempty}. We prove this theorem for $N=2$ since the proof for $N=3$ is a trivial extension of $N=2$.
\begin{theorem}
\label{finalthm}
Assuming  PD controllers are overdamped and $D_G>D_s$, this strategy ensures  (1) the robots will never fall in deadlock and (2) converge to their goals.
\end{theorem}
\hspace{-1cm}\begin{proof}
We would like to show  once phase three control begins, the robots will never fall back in deadlock. We will do this by showing  the distance between the robots is non-decreasing, once phase three control starts. 
% Recall  in deadlock, the robots are separated by the safety distance $D_s$ (\textit{i.e.} end of phase 1). Additionally, recall  phase two controller (1) rotates the assembly of the robots by $\pi$ from its initial orientation and (2) makes sure  the distance between the robots remains. Therefore, at the end of phase 2 control, the distance between the robots is still $D_s$. Now, we are interested in showing  $\norm{\Delta \boldsymbol{p}_{21}(t)}\vert_{t\geq t_2}$ is non-decreasing or alternatively, $\frac{d \norm{\Delta \boldsymbol{p}_{21}(t)}}{dt}\vert_{t\geq t_2}\geq 0$. 
\begin{figure}[t]
\centering
\subfigure[Two Robots Sims.]{\label{fig:cc}\includegraphics[trim={0.0cm 0.0cm 2.7cm 0cm},clip, width=.32\linewidth,height=3.3cm]{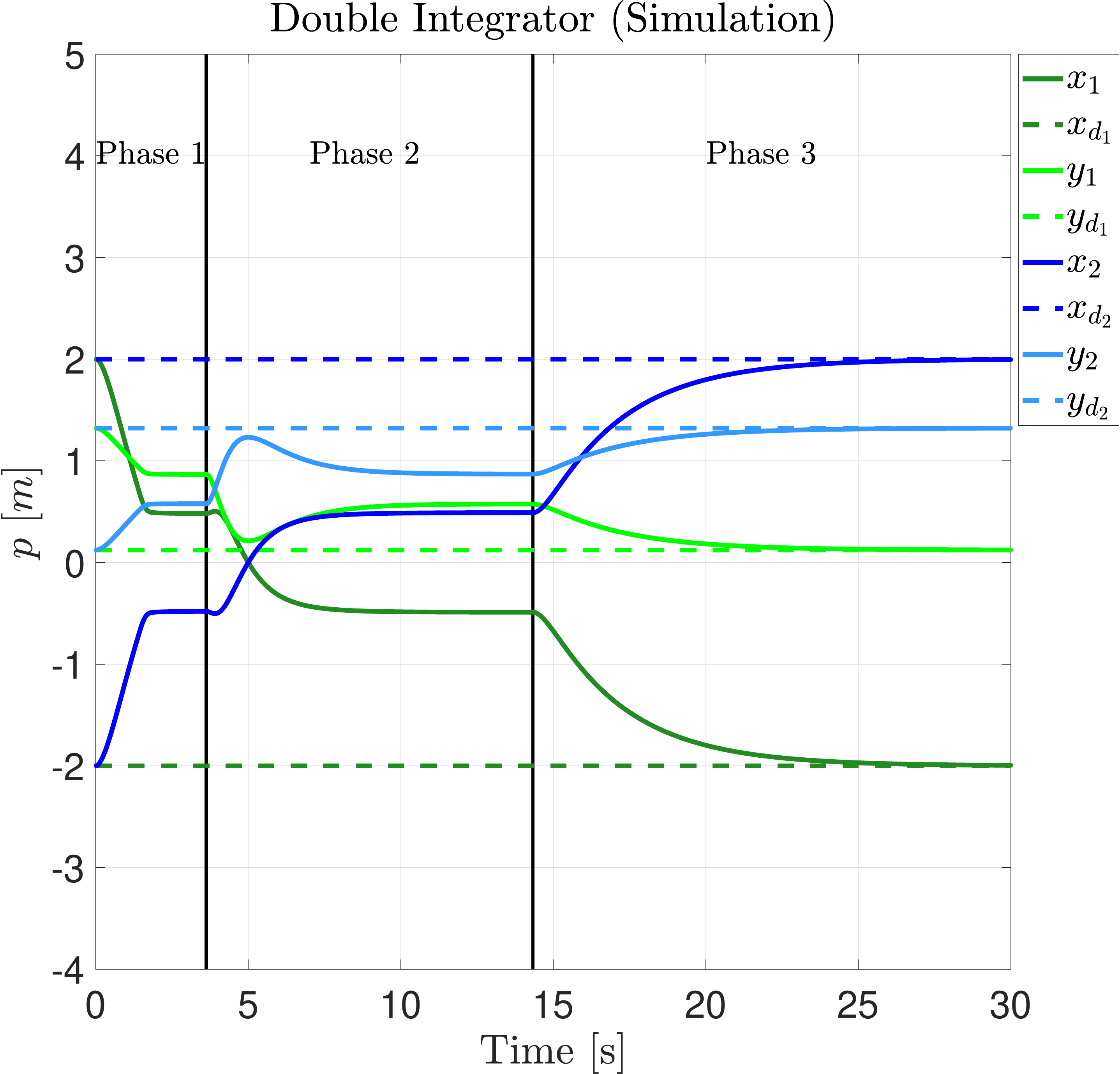}}
% \subfigure[Nonholonomic Simulation (Phase 2 using $\omega=1,v=\pm \omega D_s, t \in (t_1,t_2)$)]{\label{fig:dd}\includegraphics[trim={0.0cm 0.0cm 2.72cm 0cm},clip,width=.32\linewidth,height=3.3cm]{nonsim.eps}}
\subfigure[Expmts. with Kheperas.]{\label{fig:ee}\includegraphics[width=.32\linewidth,height=3.3cm]{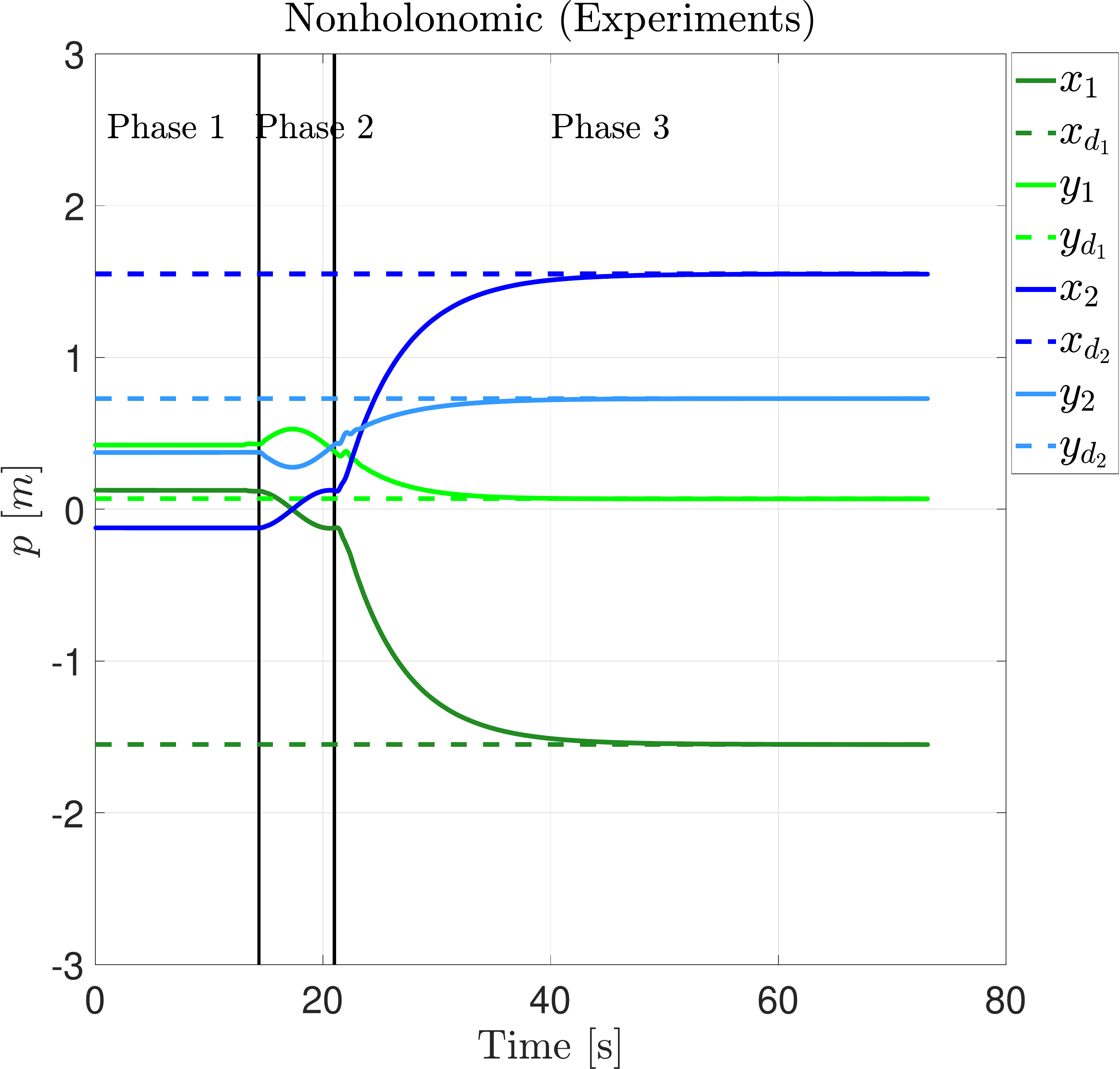}}
\subfigure[Three Robots Sims.]{\label{fig:dd}\includegraphics[trim={0.0cm 0.0cm 0.5cm 0cm},clip,width=.3\linewidth,height=3.3cm]{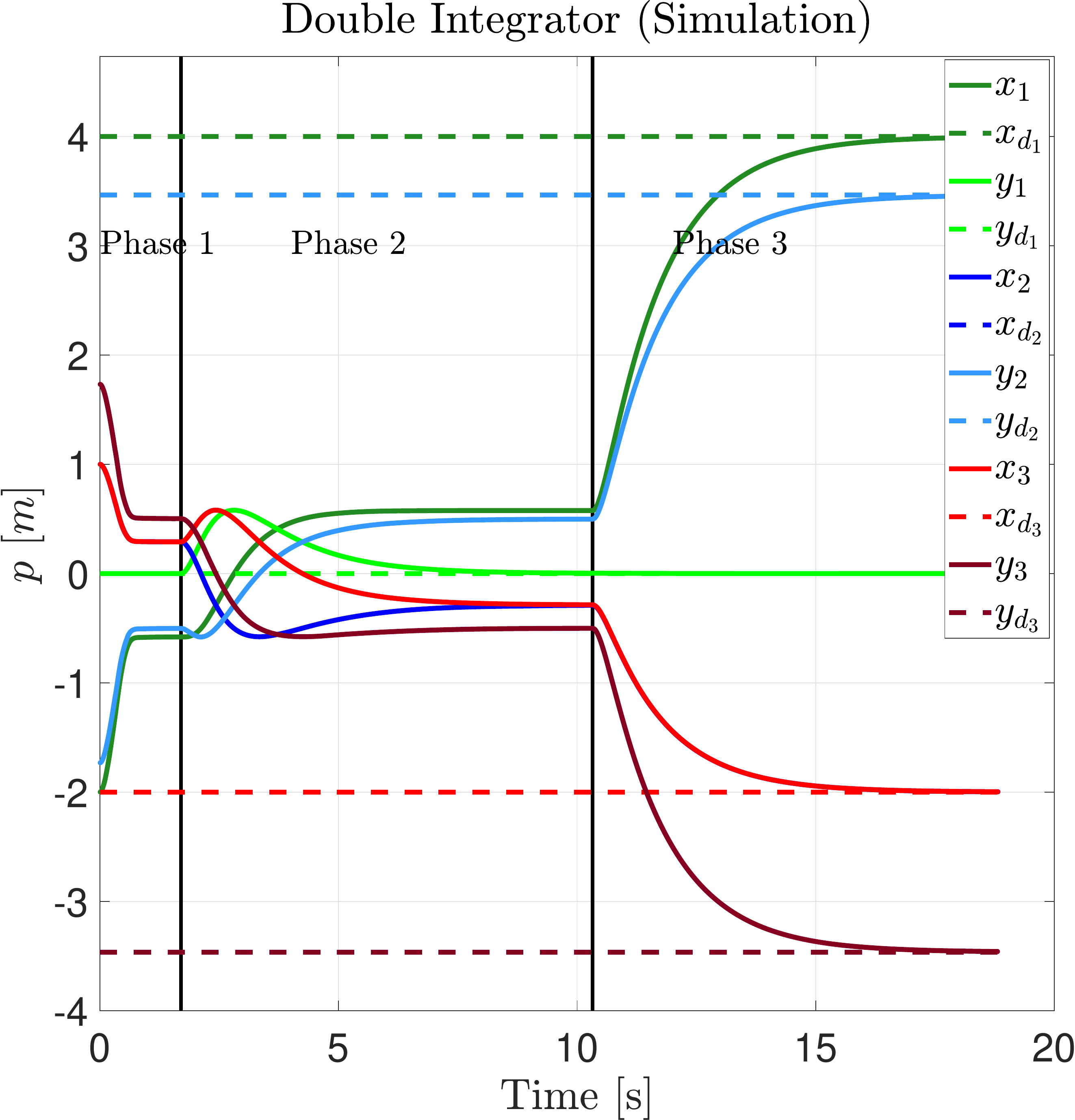}}
\setlength{\belowcaptionskip}{-15pt}
\caption{Positions of robots from deadlock resolution algorithm. In all figures, final positions converge to desired positions in Phase 3. Videos at \url{https://tinyurl.com/y4ylzwh8}}
\label{fig:simexpres}
\end{figure}
We break this proof into three parts consistent with the three phases: \\ \\
\textbf{Phase 1 $\rightarrow$ Phase 2:}
Let $t=t_1$ be the time at which phase 1 ends (and phase 2 starts) \textit{i.e.} when robots fall in deadlock. In \cref{touching} we showed  in deadlock $\norm{\Delta \boldsymbol{p}_{21}}=D_s$, and in \cref{nonempty} we showed  the positions of robots and their goals are collinear. So at the end of phase 1, $\Delta \boldsymbol{p}_{21}(t_{1})=D_s\boldsymbol{\hat{e}}_{\beta + \pi}$. The goal vector $\Delta \boldsymbol{p}_{d_{21}}(t)\coloneqq \boldsymbol{p}_{d_2}-\boldsymbol{p}_{d_1}=D_G\boldsymbol{\hat{e}}_{\beta} \mbox{ }\forall t>0$. Moreover, since the robots are static in deadlock, $\Delta \boldsymbol{v}_{21}(t_{1})=\boldsymbol{0}$\\  \\
\textbf{Phase 2 $\longrightarrow$ Phase 3:}
The initial condition of phase two is the final condition of phase one  \textit{i.e.} $\Delta \boldsymbol{p}_{21}(t_{1})=D_s\hat{e}_{\beta +\pi}$ and $\Delta \boldsymbol{v}_{21}(t_{1})=\boldsymbol{0}$. In phase two, we use feedback linearization to rotate the assembly of robots making sure  the distance between them stays at $D_s$, until the orientation of the vector $\Delta \boldsymbol{p}_{21}(t)=D_s \boldsymbol{\hat{e}}_{\theta(t)}$ aligns with  $\Delta \boldsymbol{p}_{d_{21}}=D_G\boldsymbol{\hat{e}}_{\beta}$ (\textit{see appendix for controller derivation}). Once done, $\exists$ a time $t_2$ at which $\theta(t_2)=\beta$. Moreover, at $t=t_2$, the robots are no longer moving, hence their velocities are zero, hence, $\Delta \boldsymbol{p}_{21}(t_2)=D_s \boldsymbol{\hat{e}}_{\beta},\mbox{ }\Delta \boldsymbol{v}_{21}(t_{2})=\boldsymbol{0}$. These states are the final condition for phase 2 and initial for phase 3.\\ \\
\textbf{Phase 3 $\longrightarrow \infty:$}
In this phase, the initial conditions are $\Delta \boldsymbol{p}_{21}(t_2)=D_s \boldsymbol{\hat{e}}_{\beta}$ and $\Delta \boldsymbol{v}_{21}(t_{2})=\boldsymbol{0}$. Also, note  the dynamics of phase 3 control are specified by the prescribed PD controllers. The dynamics of relative positions and velocities are:
\begin{align}
\Delta \dot{\boldsymbol{p}}_{21} &= \Delta {\boldsymbol{v}}_{21} \nonumber \\
\Delta \dot{\boldsymbol{v}}_{21} &= -k_p(\Delta \boldsymbol{p}_{21}- \Delta \boldsymbol{p}_{d_{21}}) -k_v\Delta \boldsymbol{v}_{21}, 
\end{align}
where $\Delta \boldsymbol{p}_{d_{21}}=D_G\boldsymbol{\hat{e}}_{\beta}$.  Now, we will do a coordinate change as described next. Let $\Delta \tilde{\boldsymbol{p}}_{21} \coloneqq R_{-\beta} \Delta {\boldsymbol{p}}_{21}$ and $\Delta \tilde{\boldsymbol{v}}_{21} \coloneqq R_{-\beta} \Delta {\boldsymbol{v}}_{21}$. The initial conditions in these coordinates are $\Delta \tilde{\boldsymbol{p}}_{21}(t_2)= R_{-\beta} D_s\boldsymbol{\hat{e}}_{\beta}=(D_s,0)$ and $\Delta \tilde{\boldsymbol{v}}_{21}(t_{2})=\boldsymbol{0}$ \textit{i.e.} $\Delta \tilde{p}_{21}^x(t_2)=D_s,\Delta \tilde{p}_{21}^y(t_2)=0,\Delta \tilde{v}_{21}^x(t_2)=0 \mbox{ and }\Delta \tilde{v}_{21}^y(t_2)=0$. The dynamics in new coordinates are:
\begin{align}
\Delta \dot{\tilde{\boldsymbol{p}}}_{21} &= \Delta {\tilde{\boldsymbol{v}}}_{21} \nonumber \\
\Delta \dot{\tilde{\boldsymbol{v}}}_{21} &= -k_p(\Delta \tilde{\boldsymbol{p}}_{21}- R_{-\beta}\Delta \boldsymbol{p}_{d_{21}}) -k_v\Delta \tilde{\boldsymbol{v}}_{21}.
\end{align}
Using these coordinates, note  $R_{-\beta}\Delta \boldsymbol{p}_{d_{21}}=(D_G,0)$. 
 Note from the dynamics and the initial conditions for the $y$ components of relative position and velocities  the only solution is the zero solution \textit{i.e.} $\Delta \tilde{p}_{21}^y(t) \equiv 0$ and $\Delta \tilde{v}_{21}^y(t) \equiv 0$   $\forall$ $t\geq t_2$. As for the $x$ component, we can compute the solution to be $\Delta \tilde{p}_{21}^x(t)=c_1e^{\omega_1(t-t_2)}+c_2e^{\omega_2(t-t_2)}+D_G$ and  $\Delta \tilde{v}_{21}^x(t)=c_1\omega_1e^{\omega_1(t-t_2)}+c_2\omega_2e^{\omega_2(t-t_2)}$. Here
\begin{align}
\omega_{1,2} =\frac{1}{2}\big(-k_v \pm \sqrt{k^2_v-4k_p}\big), \mbox{ }
c_1=\frac{\omega_2(D_G-D_s)}{\omega_1-\omega_2} , \mbox{ }
c_2=-\frac{\omega_1(D_G-D_s)}{\omega_1-\omega_2} \nonumber,
\end{align}
and $\omega_1-\omega_2=-\sqrt{k^2_v-4k_p}$ and $\omega_1\omega_2=k_p$. After substituting these values, we get,  $\Delta \tilde{v}_{21}^x(t)=\frac{k_p(D_s-D_G)(e^{\omega_1(t-t_2)}-e^{\omega_2(t-t_2)})}{\sqrt{k^2_v-4k_p}}$. Now, from the assumptions  PD controllers are overdamped \textit{i.e} $k_v, k^2_v-4k_p>0$ and  $D_G>D_s$, it follows   $\Delta \tilde{v}_{21}^x(t)=\frac{k_p(D_s-D_G)(e^{\omega_1(t-t_2)}-e^{\omega_2(t-t_2)})}{\sqrt{k^2_v-4k_p}} \geq 0$ and $\Delta \tilde{p}_{21}^x(t)= \frac{D_G-D_s}{\sqrt(k_v^2-4k_p)}\bigg(\omega_1e^{\omega_2(t-t_2)}-\omega_2e^{\omega_1(t-t_2)}\bigg)+D_G\geq D_s \geq 0$. Finally, note  $\frac{d \norm{\Delta \boldsymbol{p}_{12}(t)}}{dt} = \frac{\Delta \tilde{\boldsymbol{p}}^T_{21}(t)\Delta \tilde{\boldsymbol{v}}_{21}(t)}{\norm{\Delta \tilde{\boldsymbol{p}}_{21}(t)}}  \geq 0$. Hence, the distance between the robots is non-decreasing \textit{i.e.} the robots never fall in deadlock.  Additionally, since the robots use a PD-type controller, their positions exponentially stabilize to their goals. 
\end{proof} 

\section{Conclusions}

In this paper, we analyzed the characteristic properties of deadlock  results from using CBF based QPs for avoidance control in multirobot systems. We demonstrated how to interpret deadlock as a subset of the state space and proved  in deadlock, the robots are on the verge of violating safety. Additionally, we showed  this set is non-empty and bounded. Using these properties, we devised corrective control algorithm to force the robots out of deadlock and ensure task completion. We also demonstrated  the number of valid geometric configurations in deadlock increases approximately exponentially with the number of robots which makes the analysis and resolution for $N \geq 4$ complex. There are several directions we would like to explore in future. Firstly, we want to extend this to $N\geq 4$ case. In the $N=4$ case, we determined a large number of admissible geometric configurations. We find  there exist bijections among some of these configurations depending on the number of total active constraints. We believe this property can be exploited to reduce the complexity down to the equivalence classes of these bijections.  We will exploit this line of approach to simplify analysis for $N\geq 4$ cases. Secondly, we are interested in identifying the basin of attraction of the deadlock set to formally characterize all initial conditions of robots  lead to deadlock. Tools from backwards reachability set calculation can be used to compute the basin of attraction. Finally, although we focused on CBF based QPs for analysis, we will extend this to other reactive methods such as velocity obstacles and tools using value functions, and explore the properties  make a particular algorithm immune to deadlock.
\appendix
\section*{Appendix}
In this appendix, we give  further details on the family of configurations adimissible in category B of three robot deadlock. We also give the derivation of the feedback linearization controller used in Phase 2 of deadlock resolution algorithm.
%\section{Geometric Configurations Admissible in the Four Robot System Deadlock}
%In Fig. \ref{fig:deadlock_cartoon_4}, we show all the geometric configurations admissible in system deadlock. Each graph represents one valid geometric configuration. The vertex in the graph represents a robot and an edge between two vertices represents an active collision avoidance constraint. Robots connected by an edge are separated by the safety margin $D_s$. 
%\begin{figure}
%	\centering
%	\includegraphics[width=\linewidth]{FourRobotCase.jpg}
%	\setlength{\belowcaptionskip}{-15pt}
%	\caption{All Geometric Configurations Admissible in Four Robot Deadlock}
%	\label{fig:deadlock_cartoon_4}
%\end{figure}
\section{Family of states in three robot deadlock}
\begin{theorem} 
	\label{nonempty32}
	$\forall\mbox{ } k_p,k_v,D_s,R$ $>0$, $\theta \in (-\frac{\pi}{6},0)$, $\alpha \in (\frac{\pi}{6},\frac{\pi}{2})$ and $\boldsymbol{p}_{d_i}=R\hat{\boldsymbol{e}}_{2\pi (i-1)/3}$ where $i=\{1,2,3\}$, there exists a family of states in Category B, $(\boldsymbol{z}^*_1,\boldsymbol{z}^*_2,\boldsymbol{z}^*_3)_{(\theta,\alpha)} \in  \mathcal{D}_{system}$ where $\boldsymbol{z}^*_i=(\boldsymbol{p}^*_i,\boldsymbol{0})$. Here $\boldsymbol{p}^*_i$ are parametrized by $\theta,\alpha$ and are proposed as follows:  
	\begin{align}
	\boldsymbol{p}^*_1 &= \frac{-1}{2\sin{(\alpha-\theta)}} \left[\begin{matrix}
	2D_s\cos{\theta}\sin{(\alpha-\theta)}+2R\cos{\theta}\sin{(\alpha-\frac{\pi}{3})+2R\cos{\alpha}\sin{\theta}}\\
	\sin{\theta}\big(3R\sin{\alpha}+2D_s\sin{(\alpha-\theta)}-\sqrt{3}R\cos{\alpha}\big) \end{matrix} \right]\nonumber \\
	\boldsymbol{p}^*_2&=\boldsymbol{p}^*_1+D_s\hat{\boldsymbol{e}}_{\theta} \nonumber \\
	\boldsymbol{p}^*_3&=\boldsymbol{p}^*_2+D_s\hat{\boldsymbol{e}}_{\alpha}
	\end{align}
	(assuming that robot 2 has both constraints active).   
\end{theorem}
\begin{proof}
	This proof is similar to the proof of Theorem 2 so it is skipped. 
\end{proof}

\section{Derivation of Feedback Linearization Controller for Phase 2}
In the section on deadlock resolution, Phase 2 is the controller we use after Phase 1 based CBF QP once robots have fallen in deadlock. We describe the derivation of this controller here. This derivation is done for the two robot case. Extension to three robot case is trivial. 

Recall that once in deadlock, we know that the robots are exactly separated by the safety distance as was shown in Theorem 2. This means any arbitrary perturbation applied to the system can potentially cause the robots to cross the safety margin and collide with one another. Therefore, any intervention to resolve deadlock should ensure that the minimum safety margin $D_s$ is maintained. More specifically, the intervening controller must ensure that $h_{12}(t)\geq0$.  We now demonstrate how using tools from feedback linearization, we can synthesize a controller that guarantees safety. Recall the dynamics of a robot below:
\begin{align}
\left[\begin{matrix}
\dot{x} \\
\dot{y} \\
\dot{v}_{x}\\
\dot{v}_{y}
\end{matrix}\right] &= \left[\begin{matrix}
0 & 0 & 1 & 0 \\
0 & 0 & 0 & 1 \\
0 & 0 & 0 & 0\\
0 & 0 & 0 & 0
\end{matrix}\right]\left[\begin{matrix}
{x} \\
{y} \\
{v}_{x}\\
{v}_{y}
\end{matrix}\right]  + \left[\begin{matrix}
0 &0\\
0 &0 \\
1 &0\\
0 &1
\end{matrix}\right] \left[\begin{matrix}
u_{x}\\
u_{y}
\end{matrix}\right] 
\end{align}
We assume that the initial state for this robot corresponds to the time instant when the system is in deadlock ($t=t_1$). We identify two output functions which we would like to stabilize to desired values to ensure that the robots maintain the safety distance $D_s$ and rotate around each other to swap positions. The idea behind swapping positions is to ensure that at a later time, feasible controls will always exist to make the robots converge to their goals.  These output functions are described below:
\begin{enumerate}
	\item The distance between robots should not change \textit{i.e.} the robots should neither move further apart nor move closer towards one another. If we define $r(t) = \norm{\Delta \boldsymbol{p}} = \sqrt{(x_2(t)-x_1(t))^2+(y_2(t)-y_1(t))^2}$ to be the distance between the robots and $R(t) = \frac{1}{2}r^2(t)$, then we would like $y_{o_1}=\frac{dR}{dt} \longrightarrow 0$ for $\forall t \in [t_{d},t^*]$.
	\item The assembly of the two robots should rotate as a rigid body to align $\Delta \boldsymbol{p}_{12}$ with $\Delta \boldsymbol{p}_{d_{12}}$. If we define $\theta(t)=\mbox{tan}^{-1}(\frac{y_2(t)-y_1(t)}{x_2(t)-x_1(t)})$, then we would like $\theta(t) \longrightarrow \beta$ where $\beta = \mbox{tan}^{-1}(\frac{y_{d_2}-y_{d_1}}{x_{d_2}-x_{d_1}})$. This rotation of the assembly (and not individual robots) will guarantee that for each robot, there will be at-least one time instant at which the prescribed control input $\boldsymbol{\hat{u}}$ will become feasible. Whenever such a feasibility flag is turned on, one can switch the control to $\boldsymbol{\hat{u}}$ and follow it thereafter. 
\end{enumerate}
Based on first objective, we note that:
\begin{align}
y_{o_1} &=\frac{dR(t)}{dt} \nonumber \\
&= (x_2-x_1)(v_{x_2}-v_{x_1}) +  (y_2-y_1)(v_{y_2}-v_{y_1}) \nonumber \\
\dot{y}_{o_1} &= (x_2-x_1)(u_{x_2}-u_{x_1}) +  (y_2-y_1)(u_{y_2}-u_{y_1}) \nonumber \\ &+  (v_{x_2}-v_{x_1})^2 + (v_{y_2}-v_{y_1})^2 \nonumber \\
&\coloneqq -k_1y_{o_1}
\end{align}
\begin{align}
\label{obj1}
\implies &(x_2-x_1)(u_{x_2}-u_{x_1}) +  (y_2-y_1)(u_{y_2}-u_{y_1}) = \nonumber \\&\underbrace{-k_1y_{o_1} -  \{(v_{x_2}-v_{x_1})^2 
	+ (v_{y_2}-v_{y_1})^2\}}_{b_1}
\end{align}
As long as $x_2 \neq x_1$ or $y_2 \neq y_1$, we can compute controllers $\boldsymbol{u}_1=(u_{x_1},u_{y_1})$ and $\boldsymbol{u}_2=(u_{x_2},u_{y_2})$ such that $\dot{y}_{o_1}=-k_1y_{o_1}$. By selecting $k_1>>0$, we can ensure that $y_{o_1}\longrightarrow 0$ exponentially. Next, based on the second objective, note that:
\begin{align}
y_{o_2} &=\dot{\theta} \nonumber \\
& = \frac{(x_2-x_1)(v_{y_2}-v_{y_1}) -  (y_2-y_1)(v_{x_2}-v_{x_1})}{(x_2-x_1)^2+(y_2-y_1)^2} \nonumber \\
& = \frac{(x_2-x_1)(v_{y_2}-v_{y_1}) -  (y_2-y_1)(v_{x_2}-v_{x_1})}{R} \nonumber \\
\dot{y}_{o_2} &= \frac{{R[(x_2-x_1)(u_{y_2}-u_{y_1}) -  (y_2-y_1)(u_{x_2}-u_{x_1})]}}{R^2} \nonumber \\
&-\frac{\dot{R}[(x_2-x_1)(v_{y_2}-v_{y_1}) -  (y_2-y_1)(v_{x_2}-v_{x_1})]}{R^2} \nonumber \\
&= \frac{{[(x_2-x_1)(u_{y_2}-u_{y_1}) -  (y_2-y_1)(u_{x_2}-u_{x_1})]}}{R} \nonumber \\
&- \frac{y_{o_1}y_{o_2}}{R}  \nonumber \\
& \coloneqq -k_p(\theta-\beta)-k_vy_{o_2} 
\end{align}
\begin{align}
\label{obj2}
\implies &(x_2-x_1)(u_{y_2}-u_{y_1}) -  (y_2-y_1)(u_{x_2}-u_{x_1})  \nonumber \\&=  \underbrace{y_{o_1}y_{o_2}-k_pR(\theta-\beta)-k_vRy_{o_2}}_{b_2}
\end{align}
As long as $x_2 \neq x_1$ or $y_2 \neq y_1$, we can compute controllers $\boldsymbol{u}_1=(u_{x_1},u_{y_1})$ and $\boldsymbol{u}_2=(u_{x_2},u_{y_2})$ such that $\ddot{\theta}=-k_p(\theta-\beta) -k_v\dot{\theta}$. By selecting $k_p>0$, $k_v>0$ and $k_v^2-4k_p>0$, we can ensure that $\theta \longrightarrow \beta$ and $y_{o_2}=\dot{\theta}\longrightarrow 0$ exponentially. We now represent equations \cref{obj1} and \cref{obj2} in a more compact form. Denote $\delta x =(x_2-x_1)$ and $\delta y =(y_2-y_1)$
\begin{align}
\label{linearsystem1}
\left[\begin{matrix} -\delta x & -\delta y & \delta x & \delta y \\ \delta y & -\delta x & -\delta y & \delta x\end{matrix}\right] \left[\begin{matrix}
u_{x_1} \\ u_{y_1} \\ u_{x_2} \\ u_{y_2}
\end{matrix}\right] = \left[\begin{matrix}
b_1 \\ b_2
\end{matrix}\right]
\end{align}
If we additionally, impose the requirement that $\boldsymbol{u}_1 = -\boldsymbol{u}_2$, we can further reduce \cref{linearsystem1} as below:
\begin{align}
\underbrace{\left[\begin{matrix}
	-2\delta x & -2\delta y \\ 2\delta y& -2\delta x
	\end{matrix}\right]}_{\tilde{A}}\left[\begin{matrix}
u_{x_1} \\ u_{y_1} 
\end{matrix}\right] = \underbrace{\left[\begin{matrix}
	b_1 \\ b_2
	\end{matrix}\right]}_{\boldsymbol{b}} \nonumber \\
\implies \tilde{A}\boldsymbol{u}_1 = \boldsymbol{b}
\end{align}
We impose this requirement to ensure that the centroid of $\boldsymbol{p}_1,\boldsymbol{p}_2$ remains static. Thus, we can compute $\boldsymbol{u}_1 = \tilde{A}^{-1}\boldsymbol{b}$ and $\boldsymbol{u}_2 = -\boldsymbol{u}_1$. Using these controllers, we can provably guarantee rotation of the assembly of the two robots while simultaneously ensuring that the safety distance criteria is not violated. Since the distance between the robots remains unchanged, we know that \textit{i.e} $\norm{\boldsymbol{\Delta p}(t)} = D_s, \forall t \in [t_{d},\infty]$. This further implies that:
\begin{align}
\boldsymbol{\Delta p}^T(t)\boldsymbol{\Delta v}(t) = 0 \mbox{   }\forall t \in [t_{d},\infty) \nonumber \\
\implies h_{12}(\boldsymbol{z_1},\boldsymbol{z_2}) = \sqrt{2(\alpha_1 + \alpha_2)(\norm{\Delta \boldsymbol{p}_{12}}-D_{s})} \nonumber \\ 
+ \frac{\Delta \boldsymbol{p}^T_{12}\Delta \boldsymbol{v}_{12}}{\norm{\Delta \boldsymbol{p}_{12}}}
\equiv 0  \mbox{    }\forall t \in [t_{d},\infty) \nonumber \\
\implies \boldsymbol{Z}(t) \in \partial \mathcal{C}  \mbox{    } \forall t \in [t_{d},\infty) \nonumber
\end{align}
Additionally, since $\boldsymbol{u}_1(t),\boldsymbol{u}_2(t) \neq 0$, we can say that $\boldsymbol{Z}(t) \in \partial \mathcal{C} \cap \mathcal{D}^c \mbox{    }\forall t \in [t_{d},\infty)$. As a result of this controller, we know that the assembly of the robots will rotate such until $\theta$ converges to $\beta$. Once converged, we can then switch the control to the prescribed PD controllers \textit{i.e.} $\hat{\boldsymbol{u}_1}$ and $\hat{\boldsymbol{u}_2}$. These controls constitute the final phase of our three-phase control algorithm for deadlock resolution. 
Note, that in phase 1, safety was guaranteed by means of feasibility of the QP, while in phase 2, safety is guaranteed by virtue of our construction of the controller as we showed. The third phase \textit{i.e.} the prescribed PD controllers will also guarantee safety given the initial conditions that mark the start of phase 3 because the distance between the robots following the start of phase 3 is non-decreasing as shown in Theorem 6 in the paper.

\bibliographystyle{IEEEtran}
\bibliography{cmu}
\end{document}